%% file: main.tex
\documentclass[conference]{IEEEtran}
\usepackage{times}

\usepackage[numbers,sort&compress]{natbib}
\usepackage{mathtools}
\usepackage{multicol}
\usepackage[bookmarks=true,hidelinks]{hyperref}
\usepackage{multirow} 
\usepackage{xcolor}  
\usepackage{graphicx}
\usepackage{amsmath} 
\usepackage{amsthm} 
\usepackage{comment}
\usepackage{soul}
\usepackage{xspace}
\usepackage{stfloats}
\setul{0.3ex}{0.1ex} 
\usepackage[normalem]{ulem}
\useunder{\uline}{\ul}{} 
\usepackage[resetlabels]{multibib}
\newcites{app}{References} 
\usepackage{cuted}
\usepackage{caption}
\let\labelindent\relax % This clears the existing definition
\usepackage{enumitem}
\usepackage{cleveref}
\usepackage{booktabs}
\crefname{section}{Sec.}{Secs.}
\crefname{figure}{Fig.}{Figs.}
\Crefname{section}{Section}{Sections}
\crefname{proposition}{Prop.}{Props.}

\usepackage{algorithm}
\usepackage{algpseudocode}
\usepackage[font=small]{caption}

\newtheorem{definition}{Definition}
\newtheorem{assumption}{Assumption}
\newtheorem{problem}{Problem}
\newtheorem{theorem}{Theorem}   
\newtheorem{lemma}{Lemma}       

\newtheorem{proposition}{Proposition}
\newtheorem{remark}{Remark}
\usepackage{amssymb}
\usepackage{etoolbox}
\makeatletter
\patchcmd{\@makecaption}
  {\scshape}
  {}
  {}
  {}
\makeatother

\newcommand{\SE}{\mathrm{SE}}
\newcommand{\ad}{\operatorname{ad}}
\def\name{CORD\xspace}

\begin{document}

\title{Distributed Pose Graph Optimization via Continuous Riemannian Dynamics}

% You will get a Paper-ID when submitting a pdf file to the conference system
% \author{Author Names Omitted for Anonymous Review. Paper-ID [841]}

\author{\authorblockN{Jaeho Shin,
Maani Ghaffari, and
Yulun Tian}
\authorblockA{
University of Michigan\\
Ann Arbor, MI 48109, USA\\ 
Email: \{jaehos,maanigj,yulunt\}@umich.edu}}

% avoiding spaces at the end of the author lines is not a problem with
% conference papers because we don't use \thanks or \IEEEmembership

%Homer Simpson\authorrefmark{2},
%Montgomery Scott\authorrefmark{3} and
%Eldon Tyrell\authorrefmark{4}}
%Telephone: (800) 555--1212, Fax: (888) 555--1212}
%\authorblockA{\authorrefmark{4}Tyrell Inc., 123 Replicant Street, Los Angeles, California 90210--4321}}

\maketitle

\begin{abstract}
We present a framework for distributed Pose Graph Optimization (PGO) by formulating the problem as a second-order continuous-time dynamical system evolving on Lie groups.
By modeling pose variables as massive particles subject to damping, the equilibrium points of the resulting Riemannian dynamics coincide with first-order critical points of the original PGO problem.
Using the governing damped Euler--Poincar\'e equations and a semi-implicit geometric integrator, we design an optimization algorithm that generalizes existing algorithms such as Riemannian gradient descent and Gauss--Newton.
In multi-robot settings, we present a fully distributed and parallel method based on block-diagonal mass and damping matrices, 
where each robot solves an ordinary differential equation for its own poses with minimal communication overhead.
Moreover, modeling both state and velocity enables principled neighbor prediction that significantly improves convergence under delayed communication.
Theoretically, we present an analysis and establish sufficient condition that ensures energy dissipation
under the employed geometric discretization scheme.
Experiments on benchmark PGO datasets demonstrate that the proposed solver achieves superior performance compared to state-of-the-art distributed baselines in both synchronous and asynchronous regimes.
\end{abstract}

\IEEEpeerreviewmaketitle

\section{Introduction}

Multi-robot collaborative SLAM (CSLAM) is a critical capability for enabling scalable and globally consistent situational awareness in GPS-denied environments. By allowing multiple robots to jointly estimate their trajectories and build shared maps from noisy relative measurements, CSLAM significantly extends the spatial and temporal scope achievable by single-robot systems. 
Recent advances have led to robust, fielded multi-robot SLAM systems operating over large teams and long durations, even under intermittent communication and limited bandwidth \cite{ebadi_present_2024,tian2022kimera,liu_slideslam_2024,lajoie_swarm-slam_2024,schmuck_covins_2021}.
At the core of these systems lies pose graph optimization (PGO), which serves as the primary back-end for fusing intra- and inter-robot measurements into a coherent global estimate.

To address scalability, communication, and privacy constraints, a growing body of recent work has investigated distributed pose graph optimization based on distributed optimization \cite{tian_distributed_2021, mcgann_asynchronous_2024,fan_majorization_2024} or probabilistic message passing frameworks \cite{murai_robot_2024}.
Compared to centralized PGO solvers, these approaches eliminate single points of failure and allow flexible deployment across robot teams.
However, existing distributed PGO methods face persistent limitations in convergence speed and robustness. Generic solvers based on first-order primal (e.g., \cite{tian_asynchronous_2020}) or primal–dual (e.g., \cite{mcgann_asynchronous_2024}) iterations are broadly applicable but often converge slowly in large, loosely coupled graphs, particularly under asynchronous updates and stale neighbor information. 
Other state-of-the-art methods achieve acceleration by carefully exploiting problem-specific algebraic structure \cite{fan_majorization_2024}. Frequently, this design requires substantial domain expertise and tightly couples the resulting algorithm to specific problem formulations.

% \mgj{Listing contributions is better. Enumerate. Leave some white space and make it easier to read and review. Reviewers will return the favor!}
To address the aforementioned limitations, 
we propose a novel distributed PGO solver based on continuous-time Riemannian dynamics (\cref{fig:teaser}), which naturally introduces inertia and neighbor state prediction to enables accelerated convergence even under asynchronous optimization.
% grounded in continuous-time Riemannian dynamics (\cref{fig:teaser}). 
% The resulting second-order dynamics naturally introduce inertia and neighbor state prediction that enables accelerated convergence even under asynchronous optimization.
Our approach draws inspiration from the revered field of geometric mechanics~\cite{bloch1996euler}, which has played a central role in geometric control and motion planning (e.g., \cite{duong2024port,teng2025convex}) yet has so far remained largely unexplored in the context of SLAM back-end optimization.
% We reformulate PGO as a continuous-time dynamical system evolving on a product of Lie groups.
By studying the variational characterization of this system, we derive the governing equations of motion and show that appropriate choices of kinetic energy and geometric discretization lead to effective update rules that can be executed fully in a parallel and distributed manner. 
\begin{figure}[t]
    \centering
    \includegraphics[width=1\columnwidth,trim=0 100 0 10, clip]{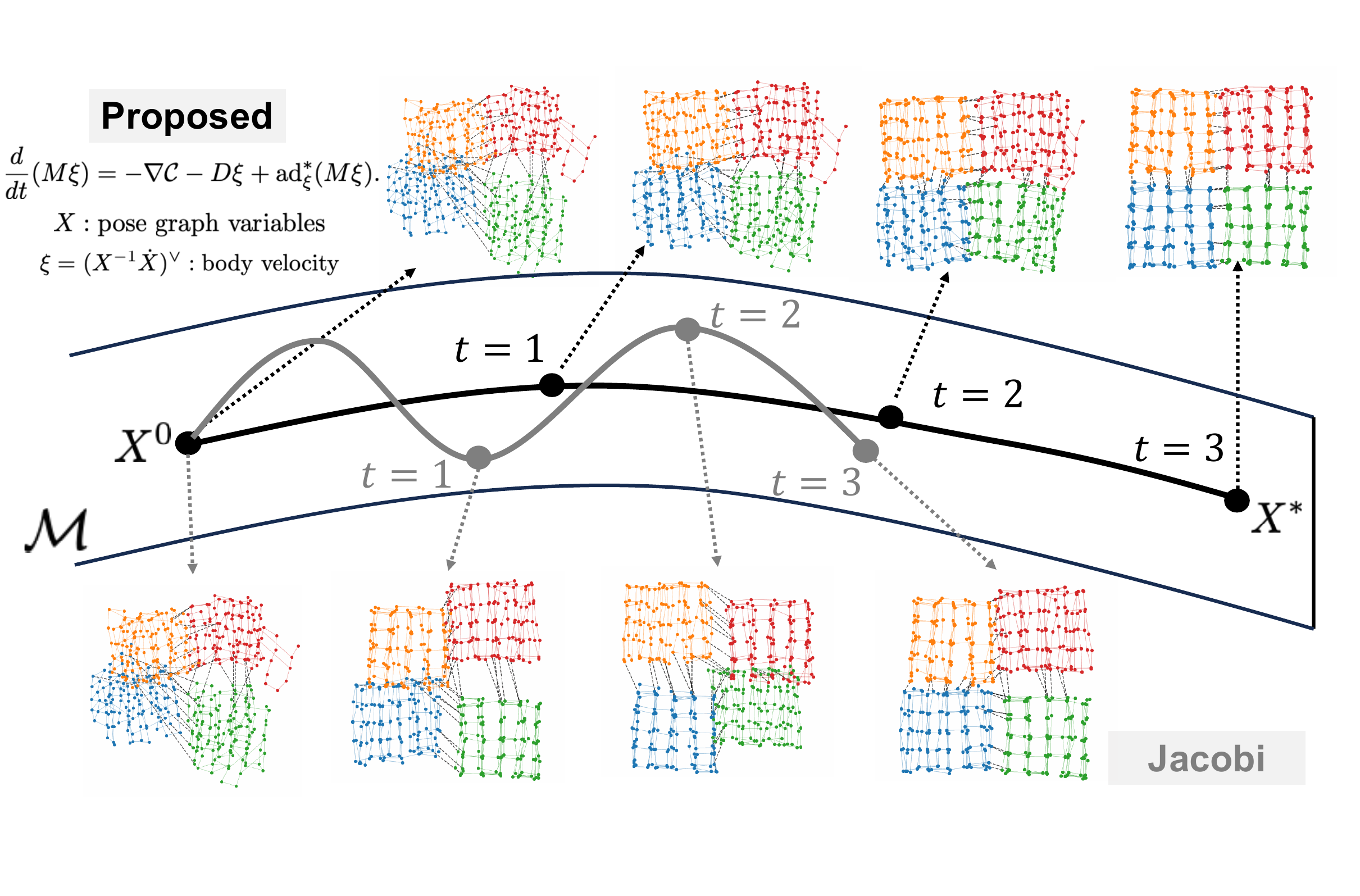}
    \caption{The proposed approach formulates pose graph optimization (PGO) as a \textbf{continuous-time dynamical system} evolving on the direct product $\mathcal{M}$ of $\SE(3)$ Lie groups governed by a damped Euler-Poincar\'{e} equation. Using a suitable geometric integration scheme, we obtain optimization trajectory (shown in solid black) that converges to a first-order critical point of PGO. 
    Each point along the trajectory represents a full pose graph estimate, where colors correspond to different robots.
    By modeling inertia and damping, the proposed approach suppresses oscillatory behavior and converges significantly faster than existing methods such as Jacobi (\textcolor{gray}{gray trajectory}).}
    \label{fig:teaser}
\end{figure}

\textbf{Contributions.} 
In summary, we present three contributions,
\begin{enumerate}[leftmargin=*,label=(\roman*)]
    \item We present \textbf{\name} (\underline{C}ontinu\underline{o}us-time \underline{R}iemannian \underline{D}ynamics), a {general distributed PGO solver} based on a second-order ordinary differential equation (ODE) evolving on Lie Groups, achieving acceleration even under delayed communication;
    \item We give a \textbf{theoretical analysis} that precisely characterizes the energy dissipation properties of the discretized Riemannian dynamics;
    \item \textbf{Extensive experimental results} show that \name achieves state-of-the-art performance in synchronous PGO with both chordal and geodesic metrics, while maintaining robustness under diverse asynchronous conditions through the proposed neighbor state prediction.
\end{enumerate}

\section{Related Work}

\textbf{Distributed Pose Graph Optimization.}
DDF-SAM~\cite{cunningham_ddf-sam_2010,cunningham_ddf-sam_2013} is a pioneering work that introduced a fully distributed SLAM back-end in which agents exchange compact summaries of their local factor graphs obtained via Gaussian elimination.
DGS~\cite{choudhary_distributed_2017} proposes a two-stage iterative optimization procedure that uses successive over-relaxation to solve chordal rotation initialization and a single Gauss-Newton update.
AMM-PGO~\cite{fan_majorization_2024} introduces a majorization–minimization framework that supports Nesterov acceleration and adaptive restart strategies in fully distributed settings.
In \cite{tian_distributed_2021}, RBCD and its accelerated variant RBCD++ performs nonlinear block-coordinate descent with first-order convergence guarantees, and attains certifiable global optimality for distributed PGO when used within a semidefinite relaxation framework \cite{rosen_se_sync_2019,briales2017cartan}.
IRBCD~\cite{li_distributed_2024} improves convergence and scalability through enhanced graph partitioning strategies.
In \cite{tian_asynchronous_2020}, Tian~et~al. develops a variant of RBCD for asynchronous optimization. The resulting ASAPP solver enjoys provable first-order convergence under bounded communication delay albeit with slower convergence speed.
Recently, primal-dual iterations based on the Alternating Direction Method of Multipliers (ADMM) has gained increasing attention \cite{mcgann_asynchronous_2024,mcgann_imesa_2024,banninger2023cross}.
MESA~\cite{mcgann_asynchronous_2024} and its incremental variant iMESA \cite{mcgann_imesa_2024} applies on-manifold consensus ADMM to distributed SLAM, and demonstrates flexible communication schedules and empirical robustness to intermittent connectivity.
Gaussian belief propagation (GBP) \cite{murai_robot_2024} proposes a probabilistic message-passing framework that enables fully asynchronous operation and naturally provides estimates of marginal uncertainty, though the convergence behavior can be sensitive to graph topology and noise characteristics.
To address this issue, recent work \cite{liu2026cbs} introduces a hybrid method that integrates maximum a-posteriori optimization with belief propagation and uses Hellinger-distance-based damping to guarantee convergence on loopy graphs.
In contrast to prior methods, we propose a unifying framework for distributed PGO based on continuous-time Riemannian dynamics. To the best of our knowledge, our approach is the first to achieve acceleration even under asynchronous communication.

% (Work in progress) Placement of this work: 
% \begin{itemize}
%     \item Propose unifying framework based on continuous-time Riemannian dynamics: recover GD, Jacobi, momentum
%     \item First accelerated algorithm under asynchronous communication, significantly faster than SOTA like MESA
% \end{itemize}

\textbf{Dynamical System Perspective in Optimization.} 
Su~et al.~\cite{su2016differential} is among the first to study Nesterov's accelerated gradient method \cite{nesterov1983method} using a continuous-time second-order ODE.
Subsequent work~\cite{wibisono2016variational,wilson2021lyapunov} generalize this perspective by introducing a Bregman Lagrangian whose Euler–Lagrange equations define a family of continuous-time dynamics, and show that appropriate discretizations of these dynamics recover existing accelerated optimization methods.
Recent methods~\cite{betancourt2018symplectic,duruisseaux2023practical} further place this framework in a Hamiltonian setting and show that symplectic integration \cite{hairer2006structure} is essential for preserving acceleration when passing from continuous to discrete time.
Parallel efforts \cite{hauswirth2016projected,duruisseaux2022accelerated} study the continuous-time dynamics perspective for optimization over Riemannian manifolds.

In robotics, \cite{golyanik2016gravitational,jauer2018efficient,golyanik2019accelerated,ali2018nrga,zhao2022graphreg, yang2021dynamical} study the related problem of point cloud registration from the dynamical system perspective.
Golyanik et al.~\cite{golyanik2016gravitational} formulate registration as a damped N-body dynamical system, where template points evolve under gravitational forces induced by the reference points.
This approach is extended in \cite{golyanik2019accelerated} that introduces a nonlinear least squares optimization over an altered potential and uses Barnes–Hut tree to accelerate point interaction.
Yang et al.~\cite{yang2021dynamical} generalizes \cite{golyanik2016gravitational,jauer2018efficient,golyanik2019accelerated} and models a broader range of pose estimation problems across point, primitive, and category-level settings.
This work extends dynamical systems–based optimization to distributed pose graph optimization, a significantly higher-dimensional and poorly-conditioned estimation problem, and explicitly accounts for communication delays that fundamentally challenge existing distributed solvers.

% (Work in progress) Placement of this work: we are inspired by and generalizes the continuous-time dynamical approach to PGO:
% \begin{itemize}
%     \item Much higher-dimensional vs. point cloud registration
%     \item A poorly conditioned problem known to be challenging to distributed optimization
%     \item Account for and resilient under communication delay
% \end{itemize}

% \begin{figure}[t]

%     \centering

%     \includegraphics[width=1\columnwidth]{example-image}

%     \caption{Dummy figure using mwe package. \color{blue} Change iteration to t = \mgj{Potentially good for the first figure.}}

% \end{figure}

\section{Problem Formulation}
\label{sec:problem}
% Unless stated otherwise, lowercase and uppercase letters denote vectors and matrices, respectively. \mgj{we use greek lowercase for matrices in Lie algebra. I used to do this but then realized global notation is not good. It's better to define each variable explicitly when used like $a \in \mathbb{R}^2$. So we can just remove this line and be explicit about variable definitions. Not using bold is also cooler and more mathy. It forces a correct and explicit definition in place instead of memorization.}
We use $\mathcal{G}$ and $\mathfrak{g}$ to denote a matrix Lie group and its associated Lie algebra.
In particular, we use $\SE(3)$ to denote the special Euclidean group in three dimensions, and $\SE(3)^n$ to denote the direct product of $n$ copies of $\SE(3)$.
% For a matrix $A$, $A^\top$ denotes its
% transpose, and $\mathrm{tr}(\cdot)$ denotes the trace operator. The identity
% and zero matrices are denoted by $I$ and $0$, with dimensions omitted when
% clear from context. 
% For a scalar-valued function $\mathcal{C}$ defined on a Riemannian manifold
% $\SE(3)^N$, $\nabla \mathcal{C}$ denotes its Riemannian gradient.\YT{Cite a Riemmanian optimization textbook.}

% \subsection{Geometric Preliminaries}
% \label{sec:problem:preliminary}

% In this section, we describe two fundamental principles governing dynamic systems evolving on matrix Lie Groups. To this end, we adopt a variational approach based on Hamilton's principle, which provides a unified framework for deriving both conservative and dissipative dynamics. 

Let $X \equiv X(t) \in \mathcal{G}$ denote a trajectory evolving on $\mathcal{G}$. We formulate the dynamics using the body velocity $\xi$, defined by left-translating the time derivative $\dot{X}$ to the Lie algebra $\mathfrak{g}$ associated with $\mathcal{G}$:
\begin{equation}
\xi := X^{-1}\dot{X} \in \mathfrak{g}.
\end{equation}

In classical mechanics, the evolution of a system is described using a
\emph{Lagrangian} $\mathcal{L} = \mathcal{T} - \mathcal{C} $, defined as the difference between kinetic energy $\mathcal{T}$ and potential
energy $\mathcal{C}$.
When the configuration space is a Lie group, this formulation extends naturally to equations of motion that evolve directly on the manifold and respect its geometric structure \cite{bloch1996euler}.
This perspective serves as the foundation of the proposed method introduced in later sections.

\subsection{Pose Graph Optimization (PGO)}
A pose graph is represented by a directed graph 
$G = (\mathcal{V}, \mathcal{E})$, 
where each node $u \in \mathcal{V}$ corresponds to an unknown robot pose 
$X_u \in \SE(3)$, and each edge $(u,v) \in \mathcal{E}$ encodes a noisy relative
pose measurement $X_{uv} \in \SE(3)$ between nodes $u$ and $v$.
The objective of PGO is to estimate the set of poses
$\{X_u\}_{u \in \mathcal{V}}$ that best agree with the relative measurements.
In this paper, we study two widely used formulations of PGO
corresponding to geodesic and chordal distance metrics on $\SE(3)$, which
together cover the objective functions underlying state-of-the-art solvers such as GTSAM~\cite{gtsam} and SE-Sync~\cite{rosen_se_sync_2019}.

\begin{problem}[Pose Graph Optimization]\label{Prob1:PGO}
\begin{equation}
\begin{aligned}
&\min_{X \in \SE(3)^{N}} 
\quad 
\mathcal{C}(X) 
\coloneq 
\sum_{(u,v) \in \mathcal{E}} 
\| r_{uv}(X_u, X_v) \|^2_{\Omega_{uv}}, \\
& r_{uv}(X_u, X_v) =
\begin{cases}
\log\!\left( X_{uv}^{-1} X_u^{-1} X_v \right), 
& \text{(Geodesic)} ,\\[4pt]
X_u X_{uv} - X_v, 
& \text{(Chordal)} .
\end{cases}
\end{aligned}
\end{equation}
Here, $X \in \SE(3)^N$ collects all $N \coloneq |\mathcal{V}|$ pose variables and $\Omega_{uv} \succ 0$ denotes the information matrix associated
with edge $(u,v) \in \mathcal{E}$. The weighted squared norm is defined as
$\| r_{uv} \|^2_{\Omega_{uv}} \coloneq r_{uv}^\top \Omega_{uv} r_{uv}$ 
for the case of the geodesic residual.
For the chordal residual, $r_{uv}$ is matrix-valued and the weighted norm is
defined as
$\| r_{uv} \|_{\Omega_{uv}}^2
\coloneq
\mathrm{tr}\!\left( r_{uv} \, \Omega_{uv} \, r_{uv}^\top \right)$, where $\Omega_{uv} = \mathrm{blkdiag}(\omega_{R_{uv}} I_3,\; \omega_{t_{uv}})$ weights the rotational and translational components \cite{briales2017cartan,rosen_se_sync_2019}.
\end{problem}

In distributed PGO, the pose graph is naturally partitioned
across robots, with each robot estimating a subset of pose variables corresponding to its own trajectory. 
Edges include both intra-robot measurements
(e.g., odometry and local loop closures) and inter-robot measurements arising from loop closures established when different robots observe the same place.
The objective in distributed PGO is to solve Problem~\ref{Prob1:PGO}
collectively and without a central coordinator, using only local computation and limited information exchange between neighboring robots; see \cref{fig:qualitative}.
% \begin{remark}
% It is straightforward to verify that the cost function $\mathcal{C}(X)$ is left-invariant~\cite{barrau2016invariant}.
% \end{remark}

\section{Proposed Method}
\label{sec:approach}

In this section, we develop the proposed \name framework for distributed PGO.
\cref{sec:approach:riemannian_dynamics} converts PGO to a continuous-time dynamical system evolving on the product of $\SE(3)$ manifolds, and show that its equilibrium states correspond to first-order critical points (FOCPs) of the PGO problem.
\cref{sec:approach:discretization} then derives the optimization algorithm resulting from geometric semi-implicit integration.
Lastly, \cref{sec:approach:distributed} specializes to fully distributed optimization that remains robust under asynchronous communication.

\subsection{From PGO to Riemannian Dynamics}
\label{sec:approach:riemannian_dynamics}
The majority of optimization algorithms for solving Problem \ref{Prob1:PGO} iteratively linearizes the residual $r_{ij}$ and obtain a descent direction toward a FOCP. 
In contrast, our algorithm first formulates a continuous-time dynamical system with a potential energy based on the PGO cost $\mathcal{C}$ and a suitably chosen kinetic energy $\mathcal{T}$.
% in which the PGO cost function $\mathcal{C}$ serves as the potential energy and we also introduce appropriate kinetic energy. 
Define the body velocity for each pose as $\xi_i := (X_i^{-1}\dot{X})^\vee \in \mathbb{R}^6$ where the operator $(\cdot)^\vee$ maps a Lie algebra element to a vector in $\mathbb{R}^6$. 
Given a generalized inertia (``mass'') matrix ${M} \in \mathbb{R}^{6N\times 6N}$ to be specified later,
the total kinetic energy is defined as $\mathcal{T} := \frac{1}{2}\xi^\top M \xi$, where $\xi = [\xi_1^\top, \cdots, \xi_N^\top]^\top$. 
Denoting the Lagrangian of the system as $\mathcal{L} := \mathcal{T}-\mathcal{C}$, we can then derive the equation of motion for the system using the Euler-Poincaré formula \cite{bloch1996euler}. 
Under the principle of least action, this system has equilibrium
configurations corresponding to FOCPs of the potential
energy $\mathcal{C}$.
Nevertheless, in the absence of dissipative forces, the total energy of the system is conserved, which results in persistent oscillations rather than convergence. To ensure convergence to the FOCP, we introduce an additional positive definite damping matrix $D \succ 0$ to enable energy dissipation.
The equation of motion can then be derived using the Lagrange-d'Alembert principle \cite{bloch1996euler}, as shown in the appendix in detail.

\begin{proposition}[Damped Euler-Poincar\'{e} Equation]
\label{prop:damped_ep}
Given positive definite mass matrix $M$ and damping matrix $D$,
the system dynamics on $\SE(3)^N$ are governed by,
    \begin{equation}\label{eq:ODE}
        \frac{d}{dt}(M\xi) = -\nabla \mathcal{C} - D\xi + \ad_{\xi}^*(M\xi).
    \end{equation}
\end{proposition}
% In \eqref{eq:ODE}, $\operatorname{ad}^*_{\xi}$ denotes the co-adjoint operator associated with $\xi$, capturing the curvature-induced coupling between momentum and velocity arising from the Lie group structure.
In~\eqref{eq:ODE}, $\ad^*_{\xi}$ denotes the
\emph{co-adjoint operator} associated with the Lie algebra element $\xi$,
defined via 
$
\langle \ad^*_{\xi}\mu, \eta \rangle
= \langle \mu, \ad_{\xi}\eta \rangle
= \langle \mu, [\xi,\eta] \rangle
$
where $[\cdot,\cdot]$ denotes the Lie
bracket on $\mathfrak{g}$.
Intuitively, this term in \eqref{eq:ODE} captures curvature-induced coupling between
velocity and momentum that arises when dynamics evolve on a Lie group
(e.g., $\SE(3)$), and it vanishes in Euclidean spaces.

In the proposed method, we allow $M$ to be either constant or state-dependent, i.e., $M \equiv M(X)$.
In the latter case, the left-hand side time derivative in \eqref{eq:ODE} expands as $\frac{d}{dt}\big(M\xi\big) = M\dot{\xi} + \dot{M}\xi$.
Moving the second term $\dot{M}\xi$ to the right-hand side of \eqref{eq:ODE}, 
\begin{equation}
M\dot{\xi} = -\nabla \mathcal{C} - D\xi + \ad_{\xi}^*(M\xi) - \dot{M}\xi.
\label{eq:ODE_Mx}
\end{equation}
Intuitively, the resulting dynamics includes a corresponding correction term to account for the varying value of $M$.
When $M$ is constant, the $\dot{M}\xi$ term also vanishes. 

\begin{remark}
We note that any steady-state solution of \eqref{eq:ODE} satisfies the FOCP condition of Problem~\ref{Prob1:PGO}: at steady state, $\dot{\xi} = \xi = 0$ so that \eqref{eq:ODE} reduces to $\nabla \mathcal{C} = 0$.
\end{remark}

% \red{TODO: We need to introduce an extension that characterizes varying $M$ and $D$; otherwise the algorithm in the next subsection is missing explanation.}

\subsection{Optimization via Geometric Discretization}
\label{sec:approach:discretization}

To numerically integrate the continuous-time dynamics in \eqref{eq:ODE}, we use a
geometric semi-implicit Euler scheme as shown in Alg.~\ref{alg:optimization}.
At each iteration, we first evaluate the induced forces according to \cref{prop:damped_ep} (lines~\ref{algline:fgrad}--\ref{algline:fvarm}). 
Specifically, line~\ref{algline:fgrad} evaluates the
potential-induced force $F_{\mathrm{grad}}=-\nabla\mathcal{C}(X_k)$ that drives the state toward a FOCP.
Line~\ref{algline:fdamp} computes the damping term needed for convergence,
and line~\ref{algline:fcor} corresponds to the co-adjoint term in \eqref{eq:ODE} that accounts for the Lie-group geometry in the momentum dynamics.
When the mass matrix is state-dependent, the additional term $F_{\mathrm{varM}}$ in line~\ref{algline:fvarm} serves as a discrete approximation of the $\dot{M}\xi$ term in \eqref{eq:ODE_Mx}. 
Given the computed forces, lines~\ref{algline:atotal}--\ref{algline:vupdate} then performs a forward Euler step on the velocity.
In lines~\ref{algline:retraction}, we use a semi-implicit scheme so that the update uses the \emph{new} velocity $\xi_{k+1}$ rather than $\xi_k$.
This choice significantly improves stability over
a fully explicit update $X_{k+1}=X_k\exp((\xi_k\Delta t)^\wedge)$ while avoiding the
nonlinear solves required by fully implicit integrators. Finally, if $M$ depends on
the current state, we refresh $M_{k+1}$ using the updated pose (line~\ref{algline:mupdate}).

% To numerically evaluate the system dynamics derived in Eq. (\ref{eq:ODE}), we employ a geometric semi-implicit Euler integration scheme. For computing the acceleration, we use the velocity of the current time step $k$, whereas the poses are updated using the updated velocity of the next time step. In this way, we ensure that the manifold retraction on $SE(3)$ utilizes the updated velocity $\xi_{k+1}$, which significantly improves the stability of the trajectory evolution compared to a standard explicit approach. Furthermore, unlike fully implicit schemes that require expensive computation at each step, our method maintains a low computational cost, thereby reducing the total wall-clock time required to reach the equilibrium. Detailed process is given as Alg. \ref{alg:optimization} \YT{Can be improved, e.g., point out which line in pseudocode correspond to semi-implicit integration, retraction, etc.}

% \red{Line 8 in Alg.~1: $M_{k+1}$ is undefined , there seems to be a circular dependency.}

\begin{algorithm}[t]
\caption{Riemannian ODE Solver}
\label{alg:optimization}
\footnotesize 
\begin{algorithmic}[1]
\Require Initial state $(X_0^i, \xi_0^i) \in \{SE(3) \times \mathbb{R}^6\}^N$, Initial mass $M_0 \in \mathbb{R}^{6N \times 6N}$, Damping $D \in \mathbb{R}^{6N \times 6N}$, Step size $\Delta t$
\State Initialize body velocity $\boldsymbol{\xi}_0 \gets \mathbf{0}_{6N}$ 
\State $k \gets 0$
\While{$k < \text{MaxIter}$}
    \State \textbf{1. Compute Forces:}
        \State \quad $F_{\text{grad}} \gets -\nabla \mathcal{C}(X_k)$ \label{algline:fgrad} 
        \State \quad $F_{\text{damp}} \gets -D\xi_k$ \label{algline:fdamp}
        \State \quad $F_{\text{cor}} \gets \text{ad}_{\xi_k}^*(M_k\xi_k)$ \label{algline:fcor}
        \State \quad $F_{\text{varM}} \gets -\left(\frac{M_{k} - M_{k-1}}{\Delta t}\right) \xi_k$ \label{algline:fvarm} %\Comment{Force due to mass variation}
    
\State \textbf{2. Numerical Integration:}
        \State \quad $F_{\text{total}} \gets F_{\text{grad}} + F_{\text{damp}} + F_{\text{cor}} + F_{\text{varM}}$ \label{algline:atotal}
        \State \quad $a_k \gets M_k^{-1} F_{\text{total}}$ \label{algline:accel}
        \State \quad $\xi_{k+1} \gets \xi_k + a_k \Delta t$ \label{algline:vupdate}
        
        \State \textbf{3. Manifold Update:}
        \State \quad $X_{k+1} \gets X_k \exp((\xi_{k+1} \Delta t)^\wedge)$ \label{algline:retraction}
        \State \quad Update Mass Matrix $M_{k+1}$ based on $X_{k+1}$ \label{algline:mupdate}
    
    \State $k \gets k + 1$
\EndWhile
\State \Return Optimized pose $X_k$
\end{algorithmic}
\end{algorithm}

Before proceeding, we present a connection between the proposed second-order dynamics and existing optimization algorithms via an overdamped limiting argument.
Specifically, consider scaling the mass matrix as $M = \varepsilon \widetilde{M}$
with $\varepsilon > 0$ and taking the limit $\varepsilon \to 0$.
In this regime, the inertial term vanishes and the body velocity becomes
$
D \xi = - \nabla \mathcal{C},
$
yielding a first-order gradient flow on the manifold.
Under an explicit time discretization with step size $\Delta t$, this induces the
update $X_{k+1} = X_k \exp(\Delta t \, \xi_k^\wedge)$. 
% \mgj{I think this $\eta$ is not the same after (3) and could cause some confusion.}
Choosing $D = I$ recovers Riemannian gradient descent \cite{absil2008optimization}.
Alternatively, choosing the damping matrix as $D = J^\top \Omega J$ yields the
Gauss--Newton direction
$
\xi = -(J^\top \Omega J)^{-1} \nabla \mathcal{C},
$
where $J$ denotes the Jacobian of the stacked residuals with respect to
the pose variables and $\Omega$ is the block-diagonal information
matrix collecting $\Omega_{ij}$ for all edges \cite{gtsam}.
Viewed from this perspective, gradient descent and Gauss--Newton arise as
\emph{first-order and overdamped limits} of the proposed second-order Riemannian dynamics.

\subsection{Distributed Riemannian Dynamics for Multi-Robot PGO}
\label{sec:approach:distributed}

\begin{comment}

* Core insight: the update rule in eq (6) can be made distributed and parallel by making M and D block-diagonal, corresponding to the robot partition
* this leads to the decomposed dynamics in (8)
* explain that the gradient term can be computed using one round of communication (exchange info for inter-robot edges only)

\end{comment}

In this section, we specialize our formulation to distributed PGO.
% The key insight enabling a fully distributed and parallel implementation of the proposed dynamics lies in the structure of the update rule in \eqref{eq:ODE}.
Specifically, by choosing the mass and damping matrices to be \emph{block-diagonal}
with respect to a robot-wise partition of the pose graph, the resulting dynamics
naturally decompose across robots while preserving the original global objective.
Concretely, we select
\(
M = \mathrm{blkdiag}(M^1,\dots,M^R)
\)
and
\(
D = \mathrm{blkdiag}(D^1,\dots,D^R),
\)
where $M^i, D^i \in \mathbb{R}^{6n_i \times 6n_i}$ correspond to robot $i$ and
$n_i$ is the number of poses owned by that robot.
Under this choice, the inertial, damping, and co-adjoint terms in \eqref{eq:ODE} are
entirely local, yielding the decomposed Riemannian dynamics for each robot $i$,
\begin{equation}
    \frac{d}{dt}(M^i \xi^i)
    = - \nabla \mathcal{C}^i
      - D^i \xi^i
      + \ad_{\xi^i}^*(M^i \xi^i).
    \label{eq:distributed_dynamics}
\end{equation}
In \eqref{eq:distributed_dynamics}, $\nabla \mathcal{C}^i$ denotes the block of the Riemannian gradient that corresponds to robot $i$'s variables.
For robot $i$ to compute $\nabla \mathcal{C}^i$, note that 
the portion of the global PGO cost involving its pose variables can be written as
\begin{equation}
   \mathcal{C}^i = \underbrace{\sum_{(u,v)\in\mathcal{E}_i^{\text{intra}}} \| r_{uv} \|^2_{\Omega_{uv}}}_{\text{Intra-robot}} + \underbrace{\sum_{j \in \mathcal{N}_i} \sum_{(u,v)\in\mathcal{E}_{ij}^{\text{inter}}} \| r_{uv}(X_u, \bar{X}_v) \|^2_{\Omega_{uv}}}_{\text{Inter-robot}},
    \label{eq:local_cost}
\end{equation}
where $\mathcal{E}_i^{\mathrm{intra}}$ denotes edges connecting poses owned by robot
$i$, $\mathcal{E}_{ij}^{\mathrm{inter}}$ denotes edges between robot $i$ and its
neighbor $j$, and $\bar{X}_v$ denotes the pose of a neighbor
robot (treated as fixed).
The Riemannian gradient block $\nabla \mathcal{C}^i$ in
\eqref{eq:distributed_dynamics} is obtained by differentiating
\eqref{eq:local_cost} with respect to robot $i$'s pose variables.
The intra-robot contribution depends only on locally available states, while the
inter-robot contribution requires pose information from neighboring robots
connected by inter-robot edges.
As a result, $\nabla \mathcal{C}^i$ can be evaluated using a single round of
neighbor communication, \emph{without} requiring global information or
centralized coordination.

\textbf{Choice of Mass and Damping Matrices}. 
We design the block-diagonal mass and damping matrices to exploit second-order information of
the PGO objective while preserving a fully distributed implementation.
Let $H = {J}^\top \Omega {J} + \lambda I$ denote the Levenberg–Marquardt approximation of the global Hessian of
$\mathcal{C}$, and let $H^i \in \mathbb{R}^{6n_i \times 6n_i}$ denote the diagonal
block of $H$ corresponding to the pose variables owned by robot $i$.
% Importantly, $H^i$ is \emph{not} formed by summing local Jacobian products, but
% is obtained by extracting the appropriate block from the global Hessian
% structure induced by the pose graph.
We define the mass and damping matrices for robot $i$ as
\(
M^i = m \, H^i
\)
and
\(
D^i = (d/t + \epsilon_d) \, H^i ,
\)
where $m, d, \epsilon_d$ are scalar coefficients. 
For $D$, the time-decaying factor $d/t$ yields behavior analogous to accelerated gradient flows \cite{su2016differential} while the small constant $\epsilon_d > 0$ preserves positive definiteness.
% Since $H^i$ is symmetric positive definite, both $M^i$ and $D^i$ are valid mass and damping matrix by construction.
With this choice, the resulting dynamics are naturally preconditioned by
$(H^i)^{-1}$, aligning the gradient-induced acceleration with a local
Levenberg–Marquardt direction.
Similar to the gradient computation, we remark that both $M^i$ and $D^i$ can be formed
in a distributed manner: since the block $H^i$ depends only on second-order contributions from intra-robot edges
and inter-robot edges incident to robot $i$, it can be assembled using local information and one round of neighbor communication.

Substituting $M^i$ and $D^i$ into \eqref{eq:distributed_dynamics}, we obtain,
% \begin{equation}
%     \dot{\xi}^i
%     =
%     -(m H^i)^{-1}
%     \bigl(
%         \nabla \mathcal{C}^i
%         - \ad_{\xi^i}^*(M^i \xi^i)
%         + \dot{M}^i \xi^i
%     \bigr)
%     - \frac{d}{m t} \, \xi^i .
%     \label{eq:preconditioned_dynamics}
% \end{equation}
\begin{equation}
    M^i\dot{\xi}^i = -\nabla \mathcal{C}^i - D^i\xi^i + \ad_{\xi^i}^*(M^i\xi^i) - \dot{M^i}\xi^i,
    \label{eq:preconditioned_dynamics}
\end{equation}
which is the analog of \eqref{eq:ODE_Mx} in the distributed setting.
As before, the dynamics in \eqref{eq:preconditioned_dynamics}
are integrated numerically using the geometric semi-implicit scheme.
Recall that the proposed framework also allows \emph{constant} choices of $M^i$ and $D^i$.
In this case, $H^i$ is computed once at the initial estimate and held fixed afterwards, 
and the $\dot{M}^i \xi^i$ term in \eqref{eq:preconditioned_dynamics} becomes zero.
In our ablation study (\cref{sec:experiment:ablation}), we show that the constant setting remains competitive for PGO, 
thus enabling a trade-off between computational cost and adaptivity.

Unlike Gauss-Seidel-style \cite{choudhary_distributed_2017,tian_distributed_2021} schemes that update blocks (robots) sequentially, 
the distributed dynamics in \eqref{eq:preconditioned_dynamics} can be integrated in a fully distributed and parallel manner. 
Under the assumption of full synchronization, each robot updates its state using the neighbor poses in the previous round. 
In contrast to sequential update, which imposes sequential dependencies that often force agents to wait, our approach eliminates such idle time ensuring that all robots compute their updates simultaneously. This independence simplifies the system architecture, and allows communication and optimization to be implemented intuitively. Moreover, the inertia and damping built in our formulation helps avoid the instability of standard Jacobi-style \cite{choudhary_distributed_2017} updates.

\begin{algorithm}[t]
\caption{Distributed Riemannian ODE Solver (Robot $i$)}
\label{alg:distributed_optimization}
\footnotesize   
\begin{algorithmic}[1]
\Require Initial state $(X_0^i, \xi_0^i)$, Mass $M^i$, Damping $D^i$, Step size $\Delta t$.
\State $k \gets 0$
\While{$k < \text{MaxIter}$}
    \State \textbf{1. Read \& Predict Neighbors:} \label{alg:distributed_optimization:predict_start}
    \For{$j \in \mathcal{N}_i$}
        \State Read latest packet $(X_{\tau}^j, \xi_{\tau}^j)$ with $\tau < t_k$.
        \State $\hat{X}_k^j \gets X_{\tau}^j \cdot \exp\left( \big(\xi_{\tau}^j (t_k - \tau)\big)^\wedge \right)$ \label{alg:distributed_optimization:extrapolate}
    \EndFor \label{alg:distributed_optimization:predict_end}

    \State \textbf{2. Compute Forces:} \label{alg:distributed_optimization:integrate_start}
    \State $F_{\text{grad}}^i \gets - \nabla \mathcal{C}_{\text{intra}}^i(X_k^i) - \sum_{j \in \mathcal{N}_i} \nabla \mathcal{C}_{\text{inter}}^{ij}(X_k^i, \hat{X}_k^j)$
    \State $F_{\text{dyn}}^i \gets - D^i \xi_k^i + \operatorname{ad}_{\xi_k^i}^*(M_k^i \xi_k^i) - \left(\frac{M_{k}^i - M_{k-1}^i}{\Delta t}\right) \xi_k^i$
    
    \State \textbf{3. Semi-Implicit Integration:}
    \State $\xi_{k+1}^i \gets \xi_k^i + (M_k^i)^{-1} (F_{\text{grad}}^i + F_{\text{dyn}}^i) \Delta t$
    \State $X_{k+1}^i \gets X_k^i \cdot \exp\left( (\xi_{k+1}^i \Delta t)^\wedge \right)$ \label{alg:distributed_optimization:integrate_end}
    \State $k \gets k + 1$
\EndWhile
\end{algorithmic}
\end{algorithm}

\textbf{Asynchronous Optimization via Delay Compensation}. 
In multi-robot SLAM, communication between robots is often subject to latency,
packet loss, and asynchronous message arrival, making it impractical to assume
access to up-to-date neighbor states at every iteration.
Thus, distributed PGO algorithms must explicitly account for delayed information when evaluating
inter-robot constraints.

To this end, we note that the proposed dynamics naturally admit a principled \emph{delay-compensation} mechanism.
Alg.~\ref{alg:distributed_optimization} presents our algorithm from the point of view of a single robot $i$.
In addition to exchanging pose estimates, each robot also transmits its body
velocity, which is cached locally upon reception.
Since the system state evolves continuously on the Lie group and is driven by
body velocities, a robot can predict a neighbor’s current pose even when the
received information is outdated.
Specifically, let $X$ and $\xi$ denote the most recently received pose and body
velocity of a neighboring robot, and let $\delta t$ denote the elapsed time since
the message timestamp.
The predicted neighbor pose is then computed as
$
    \hat{X} = X \exp\bigl( (\xi\, \delta t)^\wedge \bigr)
$, as shown in lines~\ref{alg:distributed_optimization:predict_start}-\ref{alg:distributed_optimization:predict_end}.
Using the predicted pose, the resulting preconditioned dynamics \eqref{eq:preconditioned_dynamics} are then
integrated using the geometric semi-implicit scheme (lines~\ref{alg:distributed_optimization:integrate_start}-\ref{alg:distributed_optimization:integrate_end}), 
similar to \cref{sec:approach:discretization}.
This allows all robots to advance their states in parallel without requiring synchronized communication.

\section{Convergence Analysis}

This section presents a theoretical analysis for the general Algorithm~\ref{alg:optimization} in \cref{sec:approach}.
The results also extend to the distributed Algorithm~\ref{alg:distributed_optimization} under synchronous communication as a special case. Our current analysis assumes the mass matrix $M$ to be constant.
In this case, we show that the proposed semi-implicit integrator admits a clear energy dissipation characterization, which leads to an explicit sufficient condition on the choice of step size that decreases the total system energy.
As a result, the analysis establishes a rigorous theoretical justification for the proposed dynamics and provides practical guidance for step-size selection in both centralized and distributed implementations.

The main idea is to derive the condition for the total energy of the system $E=\mathcal{T}+\mathcal{C}$ to dissipate every time step. 
We present a general analysis by building on the Lipschitz-type gradient property for pullbacks introduced in \cite{boumal2019global}. 
Boumal et al. \cite{boumal2019global} showed that this property holds on compact matrix submanifolds under mild conditions.
In \cite{tian_distributed_2021}, Tian~et~al. showed that the chordal distance formulation of PGO satisfies this condition within the sublevel set of the cost function.
% In \cite{tian_distributed_2021}, Tian~et~al. extended the analysis and derived mild conditions for PGO to satisfy this condition.
% While this property is classically guaranteed on compact submanifolds, Boumal et al. \cite{boumal2019global} demonstrated that it remains valid under mild conditions applicable to Problem \ref{Prob1:PGO}. \red{Check.}
% This generalized bound ensures that the second-order error is upper-bounded by the squared norm of the step, validating the following Riemannian descent lemma.

\begin{assumption}[Lipschitz-type gradient \cite{boumal2019global}]
\label{lemma:descent}
Let $g_{X_k}: \mathbb{R}^{6N} \to \mathbb{R}$ be the pullback of the cost function $\mathcal{C}$ defined on the tangent space at $X_k$ as:
\begin{equation}
    g_{X_k}(\eta) \coloneq \mathcal{C}(X_k \exp(\eta^\wedge)).
\end{equation}
$\mathcal{C}$ has Lipschitz-type gradient if the following holds for any update step $\eta \in \mathbb{R}^{6N}$,
\begin{equation}
    g_{X_k}(\eta) \le g_{X_k}(0) + \langle \nabla g_{X_k}(0), \eta \rangle + \frac{L}{2} \|\eta\|^2.
\end{equation}
\end{assumption}
This property states that the cost function is locally well-behaved with the second-order terms upper bounded by the squared norm of the step on the tangent space.
Given \Cref{lemma:descent}, we establish an upper bound on the total energy difference $\Delta E_k = E_{k+1} - E_k$ after a single step of integration.

\begin{lemma}[Upper Bound on Energy Change]
\label{lemma:energy_bound}
Under Assumption~\ref{lemma:descent},
the change in total energy is bounded as follows:
\begin{equation}\label{eq:DeltaE}
\begin{split}
    \Delta E_k \le & \Delta t^2(a_k^\top \nabla \mathcal{C}_k + \frac{L}{2}\left||\xi_{k+1}|\right|^2_M + \\ &\frac{1}{2}\left||a_k|\right|^2_M)  - \Delta t \xi^\top_kD\xi_k,
\end{split}
\end{equation}
where $\mathcal{C}_k \coloneq \mathcal{C}(X_k)$ and $a_k \coloneq M^{-1}(-\nabla \mathcal{C}_k -D\xi_k + \ad^{*}_{\xi_k} (M\xi_k))$. 
\end{lemma} 
The result of Lemma \ref{lemma:energy_bound} can be separated into two terms which differ by the order of multiplied $\Delta t$. 
The terms involving quadratic factors represent numerical errors arising from discretization. 
As the time step $\Delta t$ decreases, yielding more precise integration, the damping term becomes dominant. 
Indeed, divide both sides of \eqref{eq:DeltaE} by $\Delta t$ and taking the limit for $\Delta t \to 0$, 
we recover the continuous-time derivative of the total energy,
\begin{equation}
    \dot{E} = -\xi^\top D \xi < 0, \text{ when } \xi \neq 0.
\end{equation}
Thus, under the ideal continuous-time evolution, the energy strictly decays due to the positive definite damping $D$.
The central challenge is to extend such energy dissipation result to the discrete-time system.
Our key insight is that by appropriately selecting the step size $\Delta t$, the
discretization errors (on the right-hand side of~\eqref{eq:DeltaE})  can be bounded so that the damping term
dominates, yielding a net decrease in total energy.
The following theorem formalizes this intuition by deriving a sufficient condition on $\Delta t$ for which the discrete-time system preserves the energy dissipation property of the continuous-time dynamics, thereby ensuring the stability of the proposed algorithm.

% Consequently, the energy strictly decays under the positive-definite damping $D$, following the behavior of the continuous-time dynamics. Specifically, based on the definition of the total energy $E = \mathcal{K} + \mathcal{C}$, the time derivative of the energy in the continuous domain is derived as:

% which is identical to the limit obtained by dividing Eq. (5) by $\Delta t$ and letting $\Delta t \to 0$. If the derived upper bound of the energy change is negative, the discrete-time energy decreases monotonically, eventually converging to an equilibrium. Although we assumed a constant mass for analytical simplicity, the energy bound remains determinable provided that the variation of the mass matrix is bounded.

% Consequently, by enforcing the condition that the total energy must strictly decrease ($\Delta E_k < 0$), we derive a sufficient condition for the time step $\Delta t$. This constraint ensures that the physical damping dominates the numerical discretization errors, as formally stated in the following theorem:

\input{Table_Chordal.tex}

\begin{theorem}[Energy dissipation of \Cref{alg:optimization}]
\label{thm:convergence}
% Under Assumption~\ref{lemma:descent} and with constant and positive definite $M$ and $D$, \Cref{alg:optimization} is guaranteed to converge if the time step $\Delta t$ satisfies the following condition across iterations $k$:
Under Assumption~1 and with constant and positive definite $M$ and $D$, 
if the time step $\Delta t$ satisfies the following condition at each iteration $k$, 
then the total energy $E_k$ is non-increasing along the iterates of \Cref{alg:optimization}:
\begin{equation}\label{eq:dt}
    \Delta t \leq \frac{\xi_k^\top D\xi_k}{ \frac{1}{2} \| \nabla \mathcal{C}_k \|^2_{M^{-1}} + \frac{L}{2} \|\xi_{k+1}\|^2_M +\|a_k\|^2_M}.
\end{equation}
\end{theorem}

Theorem~\ref{thm:convergence} provides theoretical guidance for selecting a sufficiently small step size that preserves the energy dissipation property.
% The step-size condition in~\eqref{eq:dt} is well defined whenever $\xi_k \neq 0$, in which case the dissipation term $\xi_k^\top D \xi_k$ is strictly positive since $D \succ 0$. If $\xi_k = 0$, the numerator in~\eqref{eq:dt} vanishes and the condition becomes vacuous. 
% However, tn this case, Algorithm~\ref{alg:optimization} reduces to a first-order update driven purely by the potential gradient, since $a_k = -M^{-1}\nabla C_k$ and $\xi_{k+1} = \Delta t\, a_k$. The state update then takes the form $X_{k+1} = X_k \exp\!\big((\xi_{k+1}\Delta t)^\wedge\big)$, which corresponds to a Riemannian gradient step. Under Assumption~1, Boumal~et~al. \cite{boumal2019global} show that choosing $\Delta t < 2/L$ guarantees a decrease in the potential energy. Therefore, combining~\eqref{eq:dt} with this analysis ensures that Algorithm~\ref{alg:optimization} admits a valid step size and decreases the total energy at every iteration, including the zero-velocity case.
% To conclude this section, 
We note that the result covers the distributed algorithm in \Cref{alg:distributed_optimization} with synchronized communication  as a special case, where we use block-diagonal mass and damping matrices $(M,D)$ that decompose the global dynamics into parallel robot-wise subsystems.

% \begin{remark}[CORD Theoretical Guarantees]
% Theorem~\ref{thm:convergence} characterizes the stability of the proposed discretized dynamics and provides a sufficient condition under which the total energy is non-increasing. 
While Theorem~\ref{thm:convergence} suggests the resulting dynamics converge to a FOCP, CORD is still a local optimization solver as are AMM-PGO \cite{fan_majorization_2024} and ASAPP \cite{tian_asynchronous_2020}. 
Global optimality guarantees could be pursued by augmenting the local search with an additional distributed verification as in DC2-PGO \cite{tian_distributed_2021}.
% \end{remark}

% The convergence result in Theorem~1 applies directly to the distributed setting with the global dynamics decomposed into robot-wise blocks. 
% Under the block-diagonal choices of mass and damping matrices, each robot integrates its local dynamics using the composite gradient
% $\nabla C_i = \nabla C_i^{\mathrm{intra}} + \nabla C_i^{\mathrm{inter}}$,
% where the inter-robot term depends only on neighbor states.
% As a result, the sufficient step-size condition in~(14) can be enforced locally by each robot using its own velocity, acceleration, and gradient blocks.
% When neighbor information is exact (synchronized case), this guarantees monotonic decrease of the global energy.
% In asynchronous settings, bounded errors in $\nabla C_i^{\mathrm{inter}}$ enter as perturbations to the same condition, motivating a more conservative local step size.

% \begin{remark}[Convergence of Distributed Dynamics]
% The analysis in Theorem \ref{thm:convergence} extends directly to the distributed setting. For each robot $i$, the same time-step condition applies by substituting the global gradient with the local composite gradient $\nabla \mathcal{C}^i = \nabla \mathcal{C}_{\text{intra}}^i + \nabla \mathcal{C}_{\text{inter}}^i$. Under the assumption that $\nabla \mathcal{C}_{\text{inter}}^i$ remains bounded, satisfying this local condition guarantees the monotonic decrease of the local energy.
% \end{remark}

\section{Experiment}

In this section, we evaluate the proposed method on standard PGO benchmarks under both synchronous and asynchronous communication.
Our results demonstrate that the proposed Riemannian dynamics consistently achieve faster convergence and lower final cost than existing distributed baselines, and is particularly competitive for asynchronous optimization.

\textbf{Experiment Setup.}
Throughout the evaluation, we denote our proposed approach (Algorithm~\ref{alg:distributed_optimization}) as \name.
Following prior work \cite{tian_distributed_2021,fan_majorization_2024,mcgann_asynchronous_2024}, we generate distributed PGO problem by partitioning each pose graph dataset to simulate multiple robots (by default five); see \cref{fig:qualitative}.
Inter-robot communication is modeled using a packet-based simulation with configurable delays.
At each iteration, robots exchange pose and velocity estimates with neighbors, and delayed packets are buffered to emulate asynchronous communication.
The proposed approach is implemented in Python using PyTorch.
Consistent with prior work \cite{lian2018asynchronous, tian_asynchronous_2020}, the step size, mass, and damping parameters were determined empirically and reported in the appendix. All experiments are conducted on a workstation with Intel Core Ultra 9 285K CPU and 64GB memory.

\textbf{Metrics and Baselines.}
For quantitative evaluation, we use GTSAM \cite{gtsam} and SE-Sync \cite{rosen_se_sync_2019} to compute the centralized reference solution for the geodesic and chordal distance formulations, respectively.
Following prior work \cite{fan_majorization_2024, tian_distributed_2021}, we report performance by computing the relative optimality gap $(\mathcal{C} - \mathcal{C}^*) / \mathcal{C}^*$, where $\mathcal{C}$ and $\mathcal{C}^*$ denote the achieved and reference costs.
For the chordal distance formulation, we select AMM-PGO~\cite{fan_majorization_2024} as our baseline as it has been shown to outperform other techniques including RBCD++~\cite{tian_distributed_2021} and DGS~\cite{choudhary_distributed_2017}.
For the geodesic distance formulation, we select the state-of-the-art MESA  solver \cite{mcgann_asynchronous_2024} as the baseline.
We note that the original MESA implementation adopts an edge-based communication model, in which only a single pair of robots communicates at each iteration.
To ensure a fair comparison with our fully parallel update scheme, we modify the MESA implementation so that all robots update their local variables at every iteration.
All remaining hyperparameters for all baselines are set to the default values.
Lastly, for both the chordal and geodesic distance formulations, we also include a distributed Jacobi (DJ) baseline in which each robot performs a preconditioned gradient descent update using the diagonal block of the Hessian corresponding to its local variables. For the main asynchronous experiments, we use DJ as the primary baseline, which is equivalent to the update rule used in ASAPP~\cite{tian_distributed_2021} in this setting. 
AMM-PGO is not included because the method is designed for synchronous optimization, and extending it to delayed communication would require nontrivial modifications because stale updates interact with its Nesterov acceleration and adaptive restart mechanisms. 
In addition, MESA is not included in the main experiment as it considers a different edge-based communication regime \cite{mcgann_asynchronous_2024}.
% Although we modify MESA to use fully parallel updates in the synchronous comparison, applying our delayed-communication model to MESA is not a direct modification because its update rule is tied to a different communication protocol. 
However, in Table~\ref{tab:Edgewise_Result}, we report a separate experiment that evaluates CORD against MESA under the edge-based regime as in \cite{mcgann_asynchronous_2024}.

% As our formulation is applicable to general graph optimization on manifolds, we validated the proposed method using both Chordal and Geodesic distance metrics. The simulation assumes a synchronized global wall clock, where all robots perform optimization steps simultaneously at each iteration. Updated states are encapsulated in packets and broadcast to neighbors. To simulate physical communication latency, we implemented a delay mechanism where packets are buffered for a fixed number of steps before delivery. Consistent with prior work \cite{lian2018asynchronous, liu2015asynchronous}, the step size, mass, and damping parameters were determined empirically. For quantitative evaluation, we employed centralized solvers to establish baselines for the global optimum. Specifically, SE-Sync was used to obtain the global optimum ($f^*$) for Chordal distance experiments, allowing us to report the relative optimality gap, $(f - f^*) / f^*$. For Geodesic distance scenarios, we utilized GTSAM \cite{dellaert2012factor} to compute the reference cost of the centralized solver.

\begin{figure}[t]
    \centering
    \includegraphics[width=1\columnwidth, trim=20 20 20 0, clip]{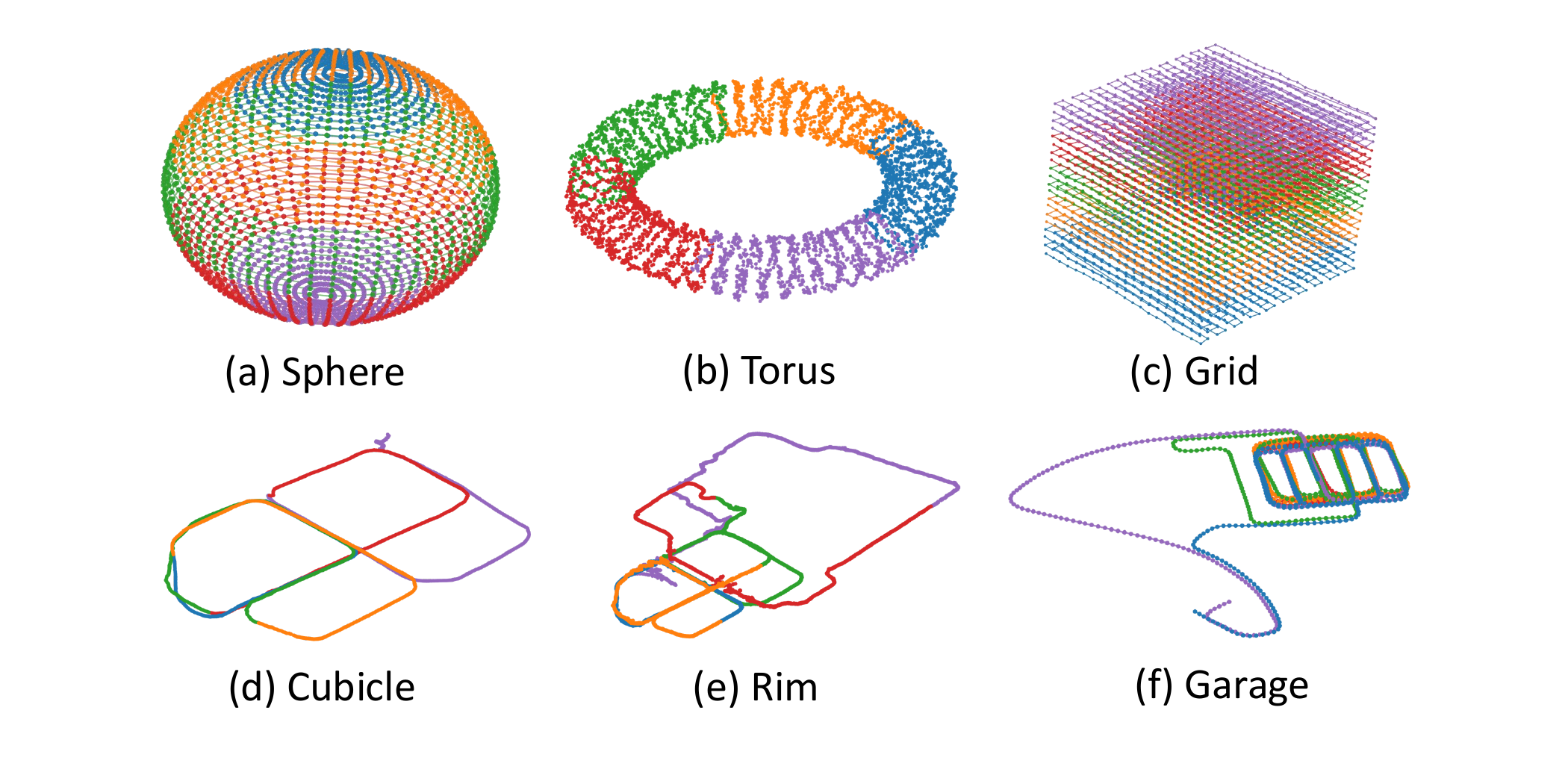}
    \caption{Trajectory estimates returned by the proposed algorithm on benchmark datasets. Different colors correspond to different robots.}
    \label{fig:qualitative}
\end{figure}

\subsection{Evaluation under Synchronous Communication}
\label{sec:experiment:sync}

In this subsection, we perform experiments under the synchronous regime without communication delay.
All methods are initialized using the same initialization scheme as implemented in AMM-PGO \cite{fan_majorization_2024}.
In Tab. \ref{tab:Chordal} and Tab. \ref{tab:Geodesic}, we report the cost achieved by each method after 100 iterations.
The costs of the initial and reference solutions are shown as $\mathcal{C}^{(0)}$ and $\mathcal{C}^\ast$, respectively.
In addition, Tab. \ref{tab:Chordal} also show the size of each dataset including the number of nodes and edges.
\name achieved the lowest error across nearly all 3D benchmark datasets for both metrics, with the exception of the \texttt{Sphere} dataset where AMM-PGO performed better, and the \texttt{Rim} dataset where MESA was better. Notably, our method outperforms AMM-PGO on the \texttt{Cubicle}, \texttt{Rim}, and \texttt{Garage}, demonstrating the validity of our algorithm on real-world data. 
Across all chordal datasets, \name consistently achieves lower error than the DJ baseline, indicating a strict improvement over first-order preconditioned updates.
This gap highlights the benefit of the proposed dynamical system formulation, where inertial and damping effects lead to faster convergence than purely gradient-based iterations.
% Furthermore, when compared to DJ, which exhibited the highest error across all Chordal datasets despite its general applicability, our results confirm that ODE is not only versatile but also demonstrates rapid convergence properties.

% All algorithms were evaluated on public PGO datasets (\texttt{SmallGrid3D, Sphere2500, torus3D, Grid3D, Cubicle, Rim, Garage}) initialized via distributed Nesterov’s accelerated chordal initialization \cite{fan_majorization_2020}. To simulate a collaborative SLAM mission involving five robots, each dataset was partitioned into five segments. 

% In the synchronous setting, we benchmarked our proposed method against state-of-the-art algorithms for Chordal and Geodesic distance minimization. Specifically, we employed AMM-PGO \cite{fan_majorization_2020} as the baseline for the Chordal metric, while MESA \cite{mcgann_asynchronous_2024} was selected as the comparative standard for Geodesic distance. Additionally, DJ \cite{choudhary_distributed_2017} was included as a general baseline applicable to both metrics to demonstrate the versatile performance of our approach across different cost functions. 

Following \cite{fan_majorization_2024},
we also report the performance profile \cite{dolan2002benchmarking}, which visualizes the percentage of problems solved by as a function of iterations.
Given tolerance $\Delta$, we introduce a threshold for determining if each problem is solved as
$
    \mathcal{C}_{\Delta} = \mathcal{C}^* + \Delta(\mathcal{C}^{(0)} - \mathcal{C}^*).
$
At each iteration, if the cost falls below the threshold, the dataset is regarded as solved by that method. 
% Consequently, the performance profile of a PGO method is defined as the percentage of problems solved by each iteration. 
The results for evaluated tolerances $\Delta=0.005$ are presented in Fig. \ref{fig:perform_plot}. We plot the performance profiles over 1000 iterations. The results indicate that \name and AMM-PGO exhibit comparable performance, whereas DJ yields the smallest area under the curve (AUC), as expected. For chordal distance, \name achieves an AUC of 851.36, comparable to AMM-PGO's 855.93. Notably, while AMM-PGO is specialized to chordal PGO, our framework is versatile to handle other cost functions including geodesic distance formulation. In the case of geodesic distance, \name significantly outperforms the baseline algorithms, with the exception of \texttt{Rim} dataset. In particular, MESA is not included in the performance profile, as we observe that with its default parameter settings the method converges prematurely across all datasets and fails to reach the prescribed accuracy threshold.

The dominant cost in solving distributed PGO is approximately solving a local pose graph using second-order information each iteration.
In CORD and DJ this appears as the Hessian inversion (Algorithm~\ref{alg:distributed_optimization}, Line~12); in AMM-PGO [8] and MESA [7] it appears as minimizing the local majorized or augmented Lagrangian objective.
% \gray{Thus, CORD has the same asymptotic per-iteration complexity.}
To quantify the practical overhead, we report average iteration runtime and communication packet size on the largest \texttt{Rim} dataset using 32-bit floats, and compare against DJ since other baselines are implemented in C++.
CORD requires 101.9 ms (vs. 43.0 ms for DJ) and 5.9 KB (vs. 3.2 KB for DJ), while using 80 iterations (vs. 924 for DJ) to achieve $\Delta=10^{-2}$. This shows a favorable trade-off since CORD achieves substantially faster convergence while incurring modest overhead.

\input{Table_Geodesic}

\begin{figure}[!t]
    \centering
    \includegraphics[width=0.98\linewidth]{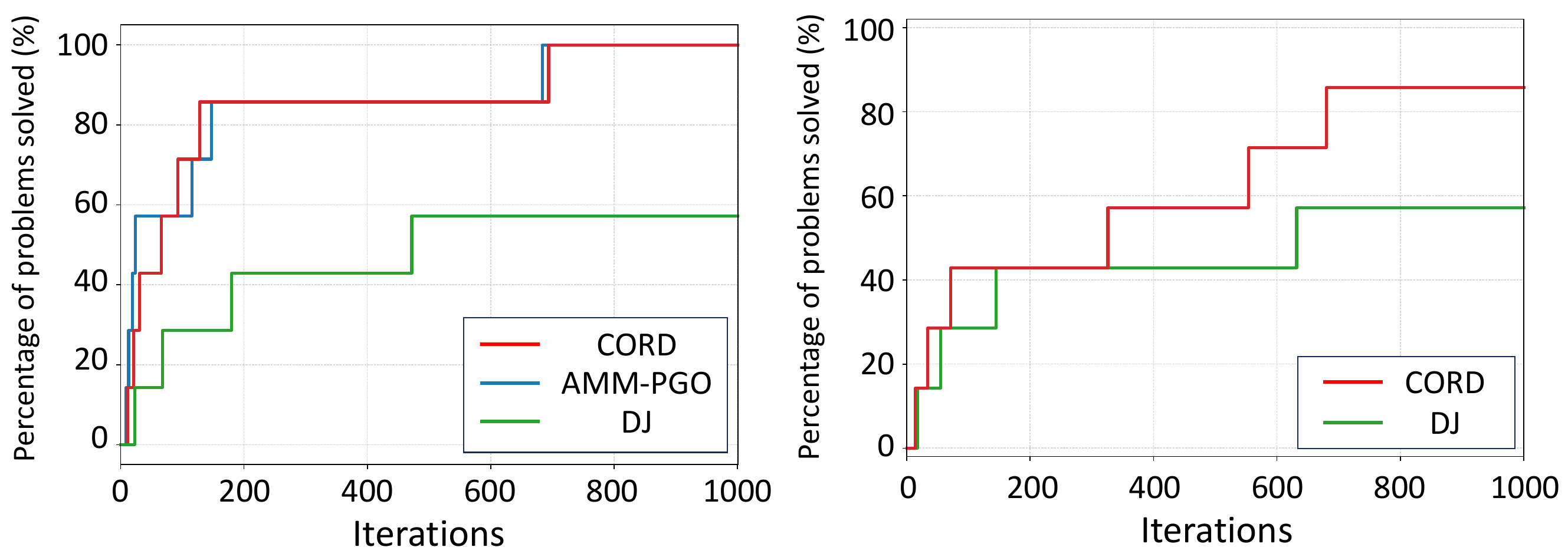}
    \caption{Performance profiles of synchronous methods over 1,000 iterations with an evaluation tolerance of $\Delta =0.005$ on 3D SLAM benchmark datasets. Left and right figures show chordal and geodesic PGO evaluations, respectively.}
    \vspace{-8pt}
    \label{fig:perform_plot}
\end{figure}

\subsection{Evaluation on Asynchronous Communication}
\label{sec:experiment:async}
We validated our approach under asynchronous conditions using both public benchmarks and simulation data. 
For the simulation, we evaluate on a different and more challenging initialization condition.
Specifically, we obtain the initial solution by propagating odometry measurements for each robot starting from its known initial pose in the global frame.
% Following Choudhary et al. \cite{choudhary_distributed_2017}, we initialize the trajectory estimates via distributed chordal initialization. Specifically, we first perform distributed rotation averaging to obtain rotation estimates. Subsequently, we fix these rotations and employ distributed Gauss-Seidel iterations to solve the reduced linear system for translations. For evaluation, we used DJ as a baseline and employed Chordal distance to measure the optimality gap relative to the global optimum obtained by SE-Sync. 

For the public benchmarks, a uniform delay of 5 steps is applied across all datasets, and results are summarized in the {``Async.''} column of Tab. \ref{tab:Chordal}. In the presence of communication delays, the use of outdated neighbor poses induces larger errors in the inter-robot edge gradients compared to the synchronized case. This error acts as a perturbation to the denominator in  \eqref{eq:dt}. Consequently, maintaining the same $\Delta t$ as in the synchronized setting would compromise convergence guarantees. To mitigate this, we employ a conservative step size by reducing $\Delta t$ by more than 90\%. Despite this restriction, \name consistently achieved a lower final cost than DJ. The inherent inertia allows the system to achieve acceleration even under delayed communication, thereby compensating for the reduced step size and facilitating faster convergence.

\begin{figure}[!t]
    \centering
    \includegraphics[width=0.98\linewidth]{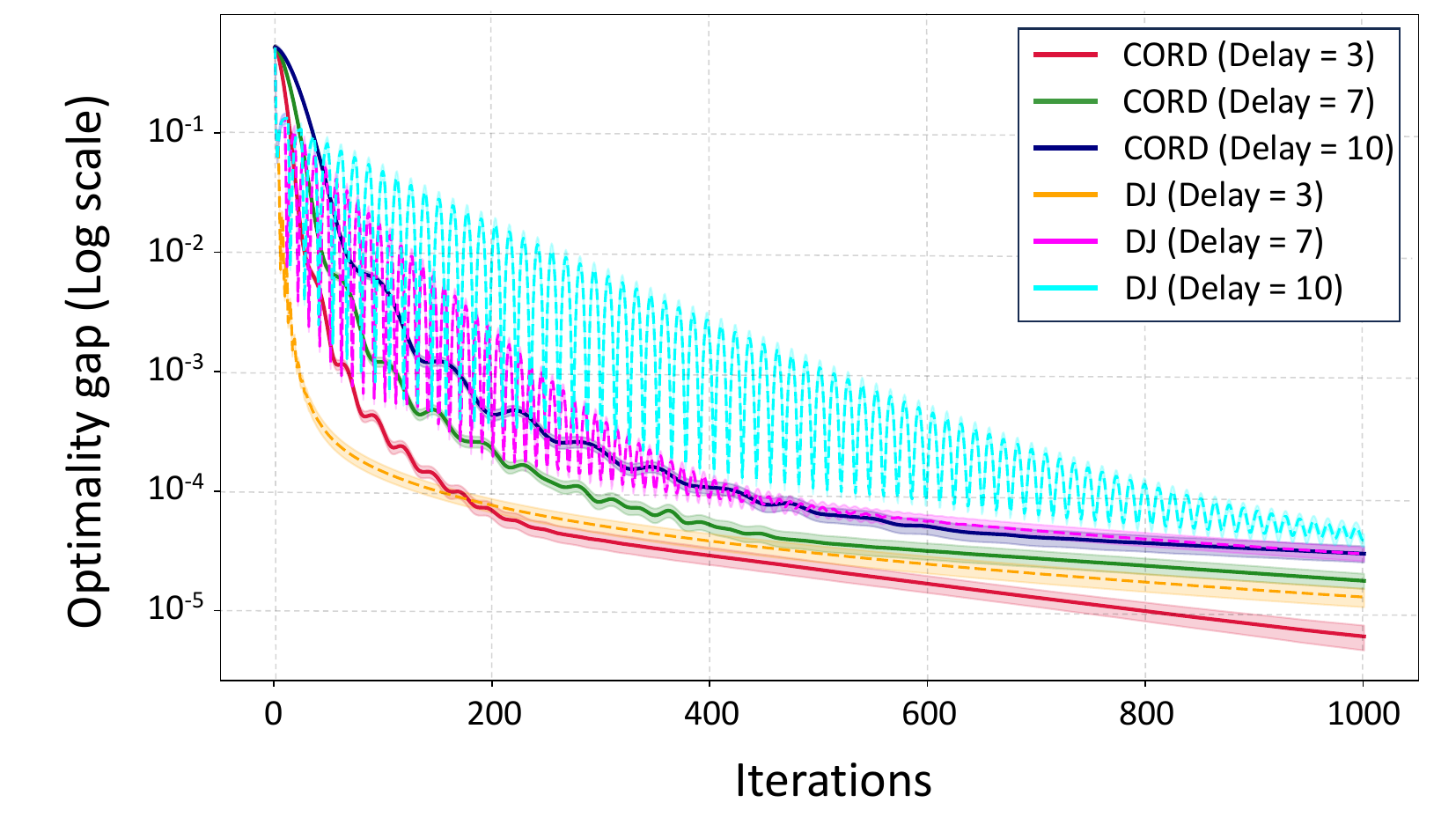}
    \caption{Optimality gap for varying delay steps in asynchronous simulations, plotted on a log scale. Solid lines indicate the mean performance averaged over 20 Monte Carlo runs, and shaded areas represent the standard deviation.}
    % \vspace{-20pt}
    \label{fig:async_plot}
\end{figure}
\input{Table_edgewise}

To provide further insights on the algorithm performance, we include additional simulation experiments with varying delay conditions.
We simulate a multi-robot SLAM scenario shown in \cref{fig:teaser} with four robots where each robot follows a 3D grid trajectory consisting of 125 nodes. 
The robots are distributed in a two by two grid formation. 
% Loop closures are generated within a 1.4 meter radius and at least two edges are guaranteed to connect each pair of robots. 
Loop closures were generated within a 1.4 m radius with a probability of 0.2 for intra-robot edges and 0.3 for inter-robot edges. 
Measurements were corrupted by Gaussian noise where standard deviations were uniformly sampled from ranges of 0.05 m to 0.15 m for translation and 1.0 degree to 3.0 degrees for rotation for intra-robot edges.
For inter-robot edges, we employ higher noise ranges of 0.10 m to 0.30 m for translation and 3.0 degrees to 10.0 degrees for rotation.
\cref{fig:async_plot} shows the optimality gap over 1000 iterations.
Regarding DJ, using the standard step size of $1.0$ causes severe oscillations and divergence. 
Similar to ASAPP \cite{tian_asynchronous_2020}, we empirically determine the maximum step size for DJ that ensured convergence through manual tuning. 
Overall, \name achieved a lower final cost than DJ across all delay steps.
When comparing performance under identical delay conditions,  \name initially exhibits a higher cost due to the lower step size. However, the DJ trajectory exhibits oscillations proportional to the delay length, leading to significant slowdown in convergence. This vulnerability stems from the fully parallel nature of the DJ update scheme, which is highly susceptible to outdated neighbor information. As communication delays increase, the discrepancy between the available neighbor poses and their actual states increases. This mismatch induces erroneously large gradients along inter-robot edges, leading to the instability of system. In contrast, even though the proposed \name framework also operates in a parallel manner, it remains robust to such perturbations due to the built-in inertia and damping of the system. Furthermore, by utilizing the communicated body velocities to predict the current states of neighbors, our method effectively compensates for communication latencies.
This predictive mechanism results in the smooth convergence trajectory observed in the results, distinct from the oscillatory behavior of DJ.

% \YT{I suggest that we remove results generated by dataset-specific tuning. Can we integrate Table IV into Table I to replace those results: row 1 in Table IV replaces the Synchronous CORD column, and row 2 can be a new column in Table I.}

% \YT{Do we have constant parameter results for CORD to update Table II (geodesic distance)?}

% \YT{The hyperparameter setting discussion and tables in the appendix can be updated or removed, since we use constant parameters now}

We further evaluate CORD under more challenging asynchronous communication conditions beyond constant delays. Specifically, we introduce $10\%$ random packet loss and randomized communication delays in the range $[1,10]$ while using the same fully parallel asynchronous update protocol and a single constant parameter setting, $d=4$, $m=0.7$, and $\Delta t=0.2$. As shown in the last column of Table~\ref{tab:Chordal}, CORD remains stable and continues to reduce the objective under these non-ideal communication conditions. These results provide additional evidence that the proposed dynamics are robust to packet loss and heterogeneous delays, beyond the constant-delay settings considered above.

\input{Table_ablation}
\subsection{Evaluation under Edge-Based Communication}
\label{sec:edge_based_eval}

To further evaluate CORD under an alternative distributed communication protocol, we adapt CORD to the edge-based regime used by MESA and compare on the three datasets considered in~\cite{mcgann_asynchronous_2024}. In this regime, one pair of neighboring robots communicates at each iteration without delay, providing a complementary setting to the fully parallel asynchronous protocol considered above. As shown in Table~\ref{tab:Edgewise_Result}, CORD achieves lower costs after 100 iterations using a single constant parameter setting, $d=4$, $m=0.7$, and $\Delta t=0.1$.

\subsection{Ablation Study}
\label{sec:experiment:ablation}

In this ablation study, we investigate the impact of different features of the proposed algorithm.
% using a position-dependent Hessian as the metric tensor and neighbor state prediction via body velocity. 
Specifically, we compare our approach against two variants: 
(i) fixing the mass and damping term $M,D$ based on the initial Hessian, 
and (ii) using received neighbor poses directly without the predictive extrapolation as in line~\ref{alg:distributed_optimization:extrapolate} of \Cref{alg:distributed_optimization}. 
For this purpose, we generate a set of datasets using the same simulation suite described in Sec. \ref{sec:experiment:async} and set a fixed communication delay of 7 steps. 
Performance was evaluated using the chordal distance optimality gap over 1000 iterations.
As shown in Tab.~\ref{tab:ablation_dict}, while the state-dependent $M$ and $D$ yield generally better performance, the improvement is overall marginal. 
This validates the constant setting as a competitive alternative that provides additional computational efficiency.
% This validates our suggestion that the choice between the two schemes represents a trade-off between computational cost and adaptivity. 
In contrast, our results show that the neighbor pose prediction plays a critical role.
Omitting the velocity-based prediction causes the system to diverge due to inaccurate neighbor information. 
This instability arises because large residuals in inter-robot edges increase the gradient $\nabla C_k$ in the denominator of \eqref{eq:dt}. 
Consequently, a small value of $\Delta t$ is needed to prevent divergence, which inevitably leads to a higher final cost.

\section{Discussion}

\subsection{Constant Mass Assumption}

While the proposed \name framework demonstrates strong performance in both synchronous and asynchronous distributed PGO settings, the current convergence analysis is restricted to the synchronous case with a constant mass matrix. Extending the theory to explicitly account for communication delays and state-dependent $M(X)$ would further close the gap between the analysis and the practical implementation.

Nevertheless, we empirically observe that the constant-mass analysis remains informative for the state-dependent variant. Near the attained minimizers, the mass matrix changes slowly, and the resulting dynamics can be viewed as a small perturbation of the analyzed system. To support this observation, Fig.~\ref{fig:MassEnergy} reports the relative mass variation,
\mbox{$\|M^k - M^{k-1}\|_F / \|M^k\|_F$},
together with the normalized energy convergence,
\mbox{$(E^k - E^*)/E^*$},
for the state-dependent setup with a constant set of parameters. After a short initial transient, the relative mass variation remains below $0.1\%$ for most iterations across all datasets, and the energy decreases monotonically except for a minor initial overshoot.

A natural next step would be to formalize both state dependence and communication delay as bounded perturbations of the current dynamics, thereby extending the convergence guarantees to more realistic distributed settings.

\subsection{Parameter Tuning}

In the current implementation, the CORD parameters are selected according to their dynamical roles: the damping coefficient $d$ controls early-stage energy dissipation, the mass $m$ affects acceleration, and $\Delta t$ determines the maximum integration step size while ensuring numerical stability. 
Although some experiments use dataset-specific parameter tuning, we observe that CORD is not highly sensitive to the choice of parameters.
In particular, a single constant setting, $d=2$, $m=0.8$, and $\Delta t=1$, achieves performance close to the per-dataset tuned setting in the synchronous chordal case, as shown in Table~\ref{tab:Chordal}. Under delayed communication, we further observe that $m$ and $d$ remain largely transferable across datasets, while $\Delta t$ should be reduced as the communication delay increases; see Table~\ref{tab:Delay_Params} in the appendix. This behavior is consistent with the theoretical insight from ASAPP~\cite{tian_asynchronous_2020}, where smaller step sizes improve stability under asynchronous updates. Adaptive parameter selection is an important direction for future work.

\subsection{Extensions to General Factor Graphs}

Although this work is formulated in the context of PGO, the underlying continuous-time Riemannian dynamics are not restricted to pairwise pose measurements.
Because our formulation builds upon the continuous-time dynamics on matrix Lie groups studied in \cite{bloch1996euler}, it inherently applies to spaces beyond SE(3), although it does not generalize to arbitrary Riemannian manifolds.
This theoretical foundation opens up an exciting direction for future work on distributed factor graph optimization \cite{gtsam} that involves optimization variables over other matrix Lie groups.

% \YT{Instead of the above, we can clarify this comment from the rebuttal: ``Because our derivation is
% based on [11], the method applies to Lie groups and not
% general Riemannian manifolds; we will clarify this throughout
% the paper.'' Optionally we can discuss generalizing to other Lie groups as future work here.}

\begin{figure}[!t]
    \centering
    \includegraphics[width=0.98\linewidth]{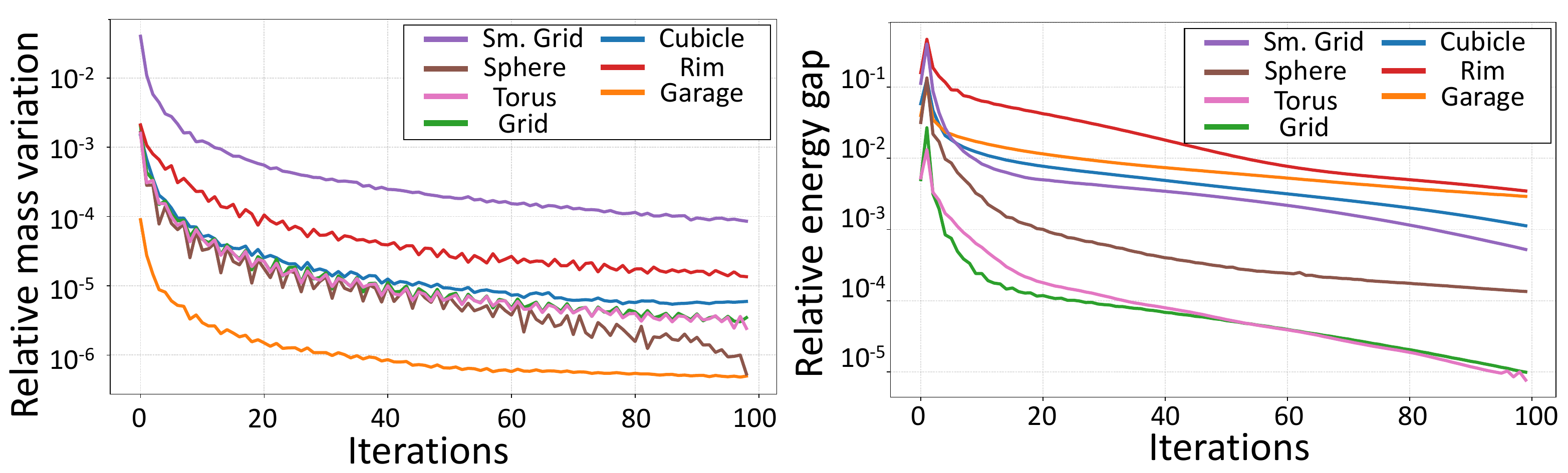}
    \captionsetup{font=footnotesize}
    \caption{Relative mass variation and energy convergence (state-dependent, synchronous, chordal) with constant parameters $d=2, m=0.8, \Delta t=1$.}

    \label{fig:MassEnergy} 
\end{figure}

\section{Conclusion} 
\label{sec:conclusion}
We presented \name, a novel framework for distributed PGO by formulating the problem as a second-order continuous dynamical system on Riemannian manifolds. 
By modeling the optimization variables as massive particles subject to dissipation, we showed that the equilibrium states of the system correspond to the first-order critical points of the original optimization problem. 
Our derived equations of motion generalize classical update rules like Gauss-Newton and Riemannian Gradient Descent. 
Furthermore, we presented a convergence analysis of the proposed algorithm, and established conditions that guarantee energy dissipation after discretization.
% Assuming practical conditions such as bounded gradients, we established the global convergence of the semi-implicit integration of the system, demonstrated its applicability to distributed environments, and showed the existence of a step size that guarantees the energy dissipation property of the continuous system.
Numerical evaluations on public benchmarks and simulation datasets validated the effectiveness of our approach. In synchronous settings, the proposed method achieved accuracy comparable to state-of-the-art distributed algorithms. 
More importantly, in asynchronous scenarios with varying communication delays, our solver demonstrated superior resilience and faster convergence speed, confirming the practical benefits of the proposed dynamical system-based solver. 
% Future work could explore adaptive mass and damping matrices to enhance conditioning and investigate the application of this dynamical framework to other geometric estimation problems beyond PGO.

\section*{Acknowledgments}
M. Ghaffari was supported by AFOSR MURI FA9550-23-1-0400 and AFOSR YIP FA9550-25-1-0224.

\clearpage
%% Use plainnat to work nicely with natbib. 
{\small
\bibliographystyle{unsrtnat}
\bibliography{references, dpgo_library}
}
\newpage 
\input{supplementary}
\end{document}

%% file: Table_Chordal.tex
\begin{table*}[t]
    \renewcommand{\arraystretch}{1.5}
    \setlength{\tabcolsep}{0.25em}
    \centering
    \caption{Optimality gap achieved by synchronous (left) and asynchronous (right) methods after 100 iterations. Best and second best methods are shown in \textbf{bold} and \underline{underline}, respectively. For CORD, we report results with both dataset-specific parameter tuning and a constant, dataset-independent parameter setting, denoted as CORD (const.).}
    \resizebox{\textwidth}{!}{%
    \begin{tabular}{|c||c|c|c|c||c|c|c|c||c|c|c|}
        \hline
        \multirow{3}{*}{Dataset} & \multirow{3}{*}{\# Nodes} & \multirow{3}{*}{\# Edges} & \multirow{3}{*}{$\mathcal{C}^{(0)}$} & \multirow{3}{*}{$\mathcal{C}^{*}$} & \multicolumn{7}{c|}{$\mathcal{C}^{(100)}$}\\
        \cline{6-12}
        & & & & & \multicolumn{4}{c||}{Synchronous} & \multicolumn{3}{c|}{Asynchronous} \\
        \cline{6-12}
        & & & & & \multirow{2}{*}{AMM-PGO \cite{fan_majorization_2024}} & \multirow{2}{*}{DJ} & \multirow{2}{*}{CORD} & \multirow{2}{*}{CORD (const.)} & \multicolumn{2}{c|}{Delay = 5} & Rand. Delay\\
        \cline{10-12}
        & & & & & & & & & DJ & CORD & CORD (const.) \\
        \hline
    {\sf Sm. Grid} & 125 & 297 & $1.5614\times 10^3$ & $1.0254\times 10^3$
     & $\mathbf{1.0254\times 10^{3}}$ & $1.0299\times 10^3$ & $\mathbf{1.0254\times 10^3}$ & \uline{$1.0256\times 10^3$} 
     & \uline{$1.0377\times 10^3$} & $\mathbf{1.0349\times 10^3}$ & $1.0313\times 10^3$ \\
        \hline
    {\sf Sphere} & 2500 & 4949 & $1.9709\times 10^3$ & $1.6870\times 10^3$
     & $\mathbf{1.6870\times 10^{3}}$ & $1.6881\times 10^3$ & \uline{$1.6872\times 10^3$} & \uline{$1.6872\times 10^3$} 
     & \uline{$1.6998\times 10^3$} & $\mathbf{1.6966\times 10^3}$ & $1.6886\times 10^3$ \\
        \hline
    {\sf Torus} & 5000 & 9048 & $2.4668\times 10^4$ & $2.4227\times 10^4$
     & $\mathbf{2.4227\times 10^{4}}$ & $2.4230\times 10^{4}$ & $\mathbf{2.4227\times 10^4}$ & $\mathbf{2.4227\times 10^4}$ 
     & \uline{$2.4257\times 10^{4}$} & $\mathbf{2.4250\times 10^4}$ & $2.4233\times 10^4$ \\
        \hline
    {\sf Grid} & 8000 & 22236 & $8.7265\times 10^4$ & $8.4319\times 10^4$
     & $\mathbf{8.4319\times 10^{4}}$ & $8.4327\times 10^{4}$ & $\mathbf{8.4319\times 10^4}$ & $\mathbf{8.4319\times 10^4}$ 
     & \uline{$8.4356\times 10^{4}$} & $\mathbf{8.4346\times 10^4}$ & $8.4328\times 10^4$ \\
        \hline
    {\sf Cubicle} & 5750 & 16869 & $8.3504\times 10^2$ & $7.1713\times 10^2$
     & $7.1799\times 10^2$ & $7.2222\times 10^2$ & $\mathbf{7.1752\times 10^2}$ & \uline{$7.1756\times 10^2$} 
     & \uline{$7.2902\times 10^2$} & $\mathbf{7.2729\times 10^2}$ & $7.2351\times 10^2$ \\
        \hline
    {\sf Rim} & 10195 & 29743 & $8.0840\times 10^3$ & $5.4609\times 10^3$
     & \uline{$5.4817\times 10^3$} & $5.6688\times 10^3$ & $\mathbf{5.4748\times 10^3}$ & $\mathbf{5.4748\times 10^3}$ 
     & \uline{$5.8761\times 10^3$} & $\mathbf{5.8424\times 10^3}$ & $5.7364\times 10^3$ \\
             \hline
    {\sf Garage} & 1661 & 6275 & $1.4175\times 10^0$ & $1.2625\times 10^0$
     & \uline{$1.2683\times 10^0$} & $1.2761\times 10^0$ & $\mathbf{1.2655\times 10^0}$ & $\mathbf{1.2655\times 10^0}$ 
     & \uline{$1.2866\times 10^0$} & $\mathbf{1.2857\times 10^0}$ & $1.2797\times 10^0$ \\
        \hline
    \end{tabular}%
    }
\label{tab:Chordal}
\vspace{-0.75em}
\end{table*}

%% file: Table_Geodesic.tex
\begin{table}[t]
    \renewcommand{\arraystretch}{1.5}
        \setlength{\tabcolsep}{0.25em}
    \centering
    \caption{Optimality gap achieved by synchronous methods after 100 iterations, under the geodesic distance formulation of PGO. Best and second best methods shown in \textbf{bold} and \underline{underline}, respectively.}  
    \resizebox{\columnwidth}{!}{
    \begin{tabular}{|c||c|c||c|c|c|}
        \hline
        \multirow{2}{*}{Dataset}&\multirow{2}{*}{$\mathcal{C}^{(0)}$} &\multirow{2}{*}{$\mathcal{C}^\ast$} & \multicolumn{3}{c|}{ $\mathcal{C}^{(100)}$}\\
        \cline{4-6}
        & & & MESA \cite{mcgann_asynchronous_2024}& DJ & CORD \\
        \hline 
    {\sf Sm. Grid} &$6.5821\times 10^{2}$  & $3.4130 \times 10^2$
     & $3.5365\times 10^{2}$ & \uline{$3.4414\times 10^{2}$} & $\mathbf{3.4320\times 10^{2}}$  \\
        \hline
    {\sf Sphere} & $7.6392\times 10^2$ & $5.7657 \times 10^2$ 
     & $5.8044\times 10^{2}$ & \uline{$5.7719\times 10^{2}$} & $\mathbf{5.7671\times 10^{2}}$  \\
        \hline
    {\sf Torus} & $9.2693 \times 10^3$ & $8.9719 \times 10^3$ 
     & $8.9770\times 10^3$ & \uline{$8.9736\times 10^3$} & $\mathbf{8.9720\times 10^3}$  \\
        \hline
    {\sf Grid} & $3.3123 \times 10^4$  & $3.0968 \times 10^4$ 
     & $3.0986\times 10^4$ & \uline{$3.0973\times 10^4$} & $\mathbf{3.0968\times 10^4}$  \\
        \hline
    {\sf Cubicle} & $3.9095 \times 10^2$ & $3.2086 \times 10^2$ 
     & \uline{$3.2337\times 10^2$} & $3.2383\times 10^2$ & $\mathbf{3.2226\times 10^2}$  \\
        \hline
    {\sf Rim} & $3.8709 \times 10^3$& $2.2367 \times 10^3$ 
     & $\mathbf{2.2476\times 10^3}$ & $2.3991\times 10^3$ & \uline{$2.2753\times 10^3$}  \\
        \hline
    {\sf Garage} & $0.7083 \times 10^{0}$& $0.6234 \times 10^{0}$ 
     & $0.6582\times 10^0$ & \uline{$0.6309\times 10^0$} & $\mathbf{0.6252\times 10^0}$  \\
        \hline
        
    \end{tabular}}
    \label{tab:Geodesic}
\vspace{-0.75em}
\end{table}

%% file: Table_edgewise.tex
\begin{table}[!t]
    \centering
    \footnotesize
    \renewcommand{\arraystretch}{1.1} 
    \setlength{\tabcolsep}{0.3em} 
    \captionsetup{font=footnotesize}
    \caption{Evaluation under edge-based regime and geodesic distance. CORD uses a constant set of parameters: $d=4$, $m=0.7$, $\Delta t=0.1$.}
    \label{tab:Edgewise_Result}
    \begin{tabular}{|l|ccc|}
        \hline
        Method & {\sf Sphere} & {\sf Torus} & {\sf Garage} \\ \hline
        MESA [7] & $7.1935 \times 10^2$ & $1.3321 \times 10^4$ & $0.7045 \times 10^0$ \\ \hline
        CORD (const.) & $6.9539 \times 10^2$ & $1.2158 \times 10^4$ & $0.6537 \times 10^0$ \\ \hline
    \end{tabular}
    \vspace{-0.75em}
\end{table}

%% file: Table_ablation.tex
\begin{table}[t]
\centering
\caption{Ablation study results reporting optimality gap. 
\texttt{State Dep.}: using state-dependent mass $M$ and damping $D$; 
\texttt{Nbr. Pred}: using predicted neighbor pose from velocity-based extrapolation.
Results are shown as Mean $\pm$ Std. Dev. across 3 datasets.}
\label{tab:ablation_dict}
\resizebox{0.8\columnwidth}{!}{ 
\begin{tabular}{ccc}
\toprule
State Dep. & Nbr. Pred. & Opt. Gap $\downarrow$ (Mean $\pm$ Std.) \\
\midrule
$\times$ & $\times$     & $1.53 \times 10^{-2} \pm 0.23 \times 10^{-2}$ \\ % Ablation 4
$\checkmark$ & $\times$ & $1.61 \times 10^{-2} \pm 0.18 \times 10^{-2}$ \\ % Ablation 3
$\times$ & $\checkmark$ & $0.29 \times 10^{-2} \pm 0.08 \times 10^{-2}$ \\ % Ablation 2
$\checkmark$ & $\checkmark$ & $\mathbf{0.28 \times 10^{-2}} \pm 0.09 \times 10^{-2}$ \\ % Ablation 1
\bottomrule
\end{tabular}%
} 
\vspace{-5mm}
\end{table}

%% file: supplementary.tex
\clearpage
\setcounter{page}{1}
\setcounter{section}{0}

\twocolumn[
    \begin{center}
        \Large
        \textbf{Distributed Pose Graph Optimization via Continuous Riemannian Dynamics}    
        \vspace{0.4cm}
        \\
        Appendix
        \vspace{0.8cm} 
    \end{center}
]

\setcounter{subsection}{0}

\section{Variational Principles on Lie Group}
In this section, we describe two fundamental principles governing dynamic systems evolving on matrix Lie Groups. To this end, we adopt a variational approach based on Hamilton's principle, which provides a unified framework for deriving both conservative and dissipative dynamics. 

Formally, the system dynamics are characterized by the standard Lagrangian $L: T\mathcal{G} \to \mathbb{R}$, defined on the tangent bundle as the difference between the kinetic energy $\mathcal{T}(X, \dot{X})$ and the potential energy $V(X)$. 
Instead of using the standard definition, we employ the reduced Lagrangian $l: \mathcal{G} \times \mathfrak{g} \to \mathbb{R}$ following \citeapp{bloch1996euler_app}.
% This provides an equivalent formulation of the dynamics, given that the kinetic energy is determined solely by the body velocity $\xi$ and is invariant to the global configuration $X$.
By rewriting the kinetic energy in the body frame as $\frac{1}{2} \langle \xi, \xi \rangle_{\mathbb{I}}$, where $\mathbb{I}$ is the symmetric positive-definite inertia tensor defining the metric on the algebra, we transition to the reduced Lagrangian $l: \mathcal{G} \times \mathfrak{g} \to \mathbb{R}$, defined as:
\begin{equation}
l(X, \xi) := \frac{1}{2} \langle \xi, \xi \rangle_{\mathbb{I}} - V(X),
\end{equation}
where $V(X)$ denotes the potential energy of the system.

Specifically, Hamilton’s principle requires considering a family of variations of the trajectory. Since the configuration space is a Lie group, we define the variation $X(t, \epsilon) \in G$ using the exponential map \citeapp{lee2007lie}:\begin{equation}X(t, \epsilon) = X(t) \exp(\epsilon \eta(t)), \label{eq:variation_def}\end{equation}where $\epsilon$ is a scalar parameter, and $\eta(t) \in \mathfrak{g}$ is an arbitrary curve in the Lie algebra (vanishing at the endpoints) representing the virtual displacement in the body frame. The variation $\delta X$ is then given by:
\begin{equation}\delta X = \frac{\partial}{\partial\epsilon} X(t, \epsilon)\Big|_{\epsilon=0} = X(t) \eta(t). \label{eq:delta_X}\end{equation} 
Based on the definition of body velocity $\xi = X^{-1} \dot{X}$ and under the standard assumption that the order of differentiating with respect of $t$ and $\epsilon$ commutes, the variation $\delta \xi$ is derived as:
\begin{equation}\begin{aligned}\delta\xi &= \frac{\partial}{\partial\epsilon} (X^{-1} \dot{X})\Big|_{\epsilon=0} \\
&= \frac{\partial X^{-1}}{\partial\epsilon}\Big|_{\epsilon=0} \dot{X} + X^{-1} \frac{\partial \dot{X}}{\partial\epsilon}\Big|_{\epsilon=0} \\&= (-X^{-1} (\delta X) X^{-1}) \dot{X} + X^{-1} \frac{\partial}{\partial t} (\delta X).\end{aligned}\end{equation}
Substituting $\delta X = X\eta$ from \eqref{eq:delta_X} into the above equation yields the following Lin constraint \citeapp{cendra1999lagrangian}:\begin{equation}\begin{aligned}\delta\xi &= -X^{-1} (X\eta) X^{-1} \dot{X} + X^{-1} \frac{\partial}{\partial t} (X\eta) \\&= -\eta (X^{-1} \dot{X}) + X^{-1} (\dot{X}\eta + X\dot{\eta}) \\&= -\eta \xi + \xi \eta + \dot{\eta} \\&= \dot{\eta} + [\xi, \eta],\end{aligned}\end{equation}where we used the identity $\dot{X} = X\xi$ and the Lie bracket $[\xi, \eta] = \xi\eta - \eta\xi$.

Now, we are ready to state Hamilton's principle, which requires the variation of the action to be zero:

\begin{definition}[Variational Principle for Conservative System \citeapp{bloch1996euler_app}]
Consider a system with configuration space $\mathcal{G}$ and reduced Lagrangian $l: \mathcal{G} \times \mathfrak{g} \to \mathbb{R}$. For a curve $X(t) \in \mathcal{G}$ with fixed endpoints, the action is stationary with respect to variations:
\begin{equation}
    \delta \mathcal{S} = \delta \int l(X, \xi) \, dt = 0.
\end{equation}
\end{definition} While Hamilton's principle governs conservative dynamics, practical systems often involve non-conservative forces such as dissipation. The Lagrange-d'Alembert principle generalizes the variational condition to incorporate these effects. The variation of the action is no longer zero but equals the virtual work done by the non-conservative forces. Let $f \in \mathfrak{g}^*$ denote the external force defined in the body frame, where $\mathfrak{g}^*$ represents the dual space of the Lie algebra. Then, the dynamics are described by the following forced variational principle.

\begin{definition}[Lagrange-d’Alembert Principle \citeapp{bloch1996euler_app}]\label{Def:LagrangeDalembert}
Consider a system with configuration space $\mathcal{G}$ and reduced Lagrangian $l: \mathcal{G} \times \mathfrak{g} \to \mathbb{R}$, subject to an external force $f \in \mathfrak{g}^*$ defined in the dual space of the Lie algebra. For a curve $X(t) \in \mathcal{G}$ with fixed endpoints, the dynamics satisfy the variational condition:
\begin{equation}\label{eq:LagrangeDalembert}
    \delta \int l(X, \xi) \, dt + \int \langle f, \eta \rangle \, dt = 0,
\end{equation}
where $\langle \cdot, \cdot \rangle$ denotes the duality pairing between the cotangent vector $f$ and the tangent vector $\eta$. 
\end{definition} In the context of matrix Lie groups, we identify the Lie algebra $\mathfrak{g}$ and the tangent space $T_X \mathcal{G}$ with their respective dual spaces via the standard trace inner product, defined as $\langle A, B \rangle := \operatorname{tr}(A^\top B)$. Consequently, the duality pairing in \eqref{eq:LagrangeDalembert} and the variational derivatives are expressed using this metric.
\section{Proofs}
\textit{Proof of Proposition} \ref{prop:damped_ep}: Let the symmetric positive-definite inertia be denoted by $\mathbb{I} =M$, the potential energy by $V(X) = \mathcal{C}(X)$, and the external force by $f = -D\xi$. The reduced Lagrangian is then given by $l = \frac{1}{2}\xi^\top M \xi -\mathcal{C}(X)$. Taking the variation of the action functional yields:
\begin{align}
    \delta S &= \int \langle\frac{\partial l}{\partial \xi}, \delta \xi\rangle + \langle\frac{\partial l}{\partial X}, \delta X\rangle dt\\ &= \int \langle M\xi, \delta \xi \rangle - \langle \frac{\partial \mathcal{C}}{\partial X}, \delta X \rangle dt \nonumber \\
    &= \int \langle M\xi, \underbrace{\dot{\eta} + [\xi, \eta]}_{\text{Lin constraint}} \rangle - \langle \frac{\partial \mathcal{C}}{\partial X}, X\eta \rangle dt \nonumber \\
    &= \int \langle M\xi, \dot{\eta} \rangle + \langle M\xi, \operatorname{ad}_\xi \eta \rangle - \langle \frac{\partial \mathcal{C}}{\partial X}, X\eta \rangle dt \nonumber \\
    &= \int \langle \underbrace{-\frac{d}{dt} (M\xi)}_{\text{Integration by Parts}}, \eta \rangle  + \langle {\operatorname{ad}^*_{\xi}(M\xi)} , \eta \rangle - \langle {X^\top \frac{\partial \mathcal{C}}{\partial X}}_{}, \eta \rangle dt.
\end{align} To obtain the final expression, we applied integration by parts to the first term using the vanishing boundary conditions of $\eta$. The second term is transformed using the definition of the co-adjoint operator, $\langle p, \operatorname{ad}_\xi \eta \rangle = \langle \operatorname{ad}^*_\xi p, \eta \rangle$. Finally, the third term utilizes the matrix trace identity $\langle A, XB \rangle = \langle X^\top A, B \rangle$ to isolate $\eta$. By including the virtual work done by the external force, we obtain:

\begin{align}
    &\int \langle-\frac{d}{d t} (M\xi), \eta \rangle  + \langle \text{ad}^*_{\xi}(M\xi) , \eta \rangle - \langle X^\top \frac{\partial C}{\partial X}, \eta \rangle  + \langle f,\eta\rangle dt   \\ = &\int \langle -\frac{d}{d t} (M\xi)+ \text{ad}^*_{\xi}(M\xi) - X^\top \frac{\partial C}{\partial X}-D\xi,\eta\rangle dt \nonumber \\ 
    = & 0. \nonumber 
\end{align} 
Since the variation $\eta$ is arbitrary, the integration must vanish. 

% Identifying the term $X^\top \frac{\partial \mathcal{C}}{\partial X}$ as the Riemannian gradient of the cost function, denoted by $\nabla \mathcal{C} \triangleq \frac{\partial\mathcal{C}}{\partial\xi}$ (corresponding to the gradient with respect to the right perturbation),
To rigorously identify the term $X^\top \frac{\partial \mathcal{C}}{\partial X}$ as the Riemannian gradient $\nabla \mathcal{C}$, we examine the time derivative of the cost function along a trajectory $X(t)$. Recall that the body velocity is defined as $\xi = X^{-1} \dot{X} \in \mathfrak{g}$. By the definition of the gradient on the Lie algebra, the rate of change of the cost function is given by the inner product of the gradient and the velocity:
\begin{equation}
    \frac{d}{dt} \mathcal{C}(X(t)) =  \lim_{h \to 0} \frac{\mathcal{C}( X \cdot \exp(t \xi) ) - \mathcal{C}(X)}{t} = \langle \nabla \mathcal{C}, \xi \rangle . \label{eq:time_deriv_def}
\end{equation} Alternatively, applying the chain rule directly in the ambient matrix space yields: \begin{align} \label{eq:ambient_chain}
    \frac{d}{dt} \mathcal{C}(X(t)) =&  \left\langle \frac{\partial \mathcal{C}}{\partial X}, \dot{X} \right\rangle\\ =& \left\langle \frac{\partial \mathcal{C}}{\partial X}, X \xi \right\rangle \nonumber \\ 
    = &  \left\langle X^\top \frac{\partial \mathcal{C}}{\partial X}, \xi \right\rangle. \nonumber
\end{align} Comparing \eqref{eq:time_deriv_def} and \eqref{eq:ambient_chain}, which must hold for any velocity $\xi$, we identify the gradient as:
\begin{equation}
    \nabla \mathcal{C} = X^\top \frac{\partial \mathcal{C}}{\partial X}.
\end{equation} As a result, we arrive at the damped Euler-Poincaré equation:
\begin{equation}
    \frac{d}{dt}(M\xi) = -\nabla \mathcal{C} - D\xi + \ad_{\xi}^*(M\xi). 
\end{equation}  \qed

\textit{Proof of Lemma} \ref{lemma:descent}: Let us define $\Delta T_k:= T_{k+1} - T_k$ and $\Delta \mathcal{C}_k := \mathcal{C}_{k+1} - \mathcal{C}_{k}$. Using Assumption \ref{lemma:descent} with $\eta = \xi_{k+1}\Delta t$, we have:
\begin{align}\label{eq:DeltaC}
    \Delta \mathcal{C}_k &= g_{X_k}(\xi_{k+1} \Delta t) -g_{X_k}(0)\\ &\le \langle \nabla g_{X_k}(0) , \xi_{k+1}\Delta t \rangle + \frac{L}{2} \|\xi_{k+1} \Delta t\|^2_M \nonumber \\ 
    &= \Delta t\xi_{k+1}^\top\nabla \mathcal{C}_k + \frac{L \Delta t^2}{2} \|\xi_{k+1}\|_M^2 \nonumber 
\end{align} Furthermore, the change in kinetic energy is given by:

\begin{align}\label{eq:DeltaT}
    \Delta T_k =& \frac{1}{2}\xi_{k+1}^\top M\xi_{k+1} - \frac{1}{2}\xi_k^\top M\xi_k \\ 
    =& \frac{1}{2}(\xi_k ^\top + \Delta t a_k^\top) M(\xi_k + \Delta t a_k) - \frac{1}{2}\xi_k^\top M\xi_k \nonumber \\ 
    =& \Delta t \xi_k^\top M a_k + \frac{\Delta t ^2}{2} \|a_{k}\|_M^2 \nonumber \\ =& \Delta t\xi_k^\top(-\nabla\mathcal{C}_k- D\xi_k) + \frac{\Delta t ^2}{2} \|a_{k}\|_M^2 \nonumber 
\end{align} We obtain the final expression since the co-adjoint term included in $Ma_k$ vanishes:
\begin{align}
\xi^\top_k \text{ad}^*_{\xi_k}(M\xi_k) =& \langle  \text{ad}^*_{\xi_k}(M\xi_k),\xi_k  \rangle \\  \nonumber
    =& \langle M\xi_k, \text{ad}_{\xi_k}(\xi_k)\rangle \\ \nonumber 
    =& \langle M\xi_k, [\xi_k,\xi_k]\rangle \\ \nonumber
    =& 0.
\end{align} Summing \eqref{eq:DeltaC} and \eqref{eq:DeltaT} yields:
\begin{align}\label{eq:DeltaE_proof}
    \Delta E_k =& \Delta T_k + \Delta\mathcal{C}_k \\ 
    \le& \Delta t(\xi^\top_{k+1} -\xi^\top_k)\nabla\mathcal{C}_k +\Delta t^2(\frac{L }{2} \|\xi_{k+1}\|_M^2 + \frac{1}{2} \|a_{k}\|_M^2) \nonumber \\ &-\Delta t \xi^\top_k D\xi_k \nonumber \\ \nonumber 
    =& \Delta t^2(a_k^\top \nabla \mathcal{C}_k + \frac{L }{2} \|\xi_{k+1}\|_M^2 + \frac{1}{2} \|a_{k}\|_M^2) - \Delta t\xi^\top_k D \xi_k \nonumber 
\end{align} \qed

\textit{Proof of Theorem} \ref{thm:convergence}: Recalling that $\| \cdot \|_M$ denotes the norm induced by the metric $M$ (i.e., $\|x\|_M = \sqrt{x^\top M x}$), the generalized Cauchy-Schwarz and Young's inequalities yield:
\begin{equation}
    a^\top_k \nabla C_k \le \|a_k\|_M\|\nabla \mathcal{C}_k\|_{M^{-1}} \le \frac{1}{2}\|a_k\|^2_M + \frac{1}{2} \|\nabla \mathcal{C}_k \|^2_{M^{-1}}.
\end{equation}  Consequently, using \eqref{eq:DeltaE_proof}, $\Delta E_k$ is bounded by:
\begin{equation}
\begin{aligned}
    \Delta E_k \le&  \Delta t^2(a_k^\top \nabla \mathcal{C}_k + \frac{L }{2} \|\xi_{k+1}\|_M^2 + \frac{1}{2} \|a_{k}\|_M^2) - \Delta t\xi^\top_k D \xi_k  \\  
    \le& \Delta t^2( \frac{1}{2}\|\nabla \mathcal{C}_k \|^2_{M^{-1}} + \frac{L }{2} \|\xi_{k+1}\|_M^2 +  \|a_{k}\|_M^2)  \\ &- \Delta t\xi^\top_k D \xi_k 
    \\\coloneq & B_k 
\end{aligned}    
\end{equation}
    
 A sufficient condition for energy dissipation (i.e., $\Delta E_k \le 0$) is $B_k \le 0$. Solving this inequality for $\Delta t$, we obtain:

\begin{equation}
    \Delta t \leq \frac{\xi_k^\top D\xi_k}{ \frac{1}{2} \| \nabla \mathcal{C}_k \|^2_{M^{-1}} + \frac{L}{2} \|\xi_{k+1}\|^2_M +\|a_k\|^2_M}.
\end{equation} \qed
\section{Experimental Details}

\raggedbottom
In this section, we provide the detailed hyperparameter configurations used in our experiments to ensure reproducibility. Specifically, we list the values for the damping coefficient ($d$), mass ($m$), and discretization time step ($\Delta t$) across different datasets and evaluation metrics. 
As discussed in the main paper, these parameters are selected to balance the convergence speed and numerical stability of the Riemannian dynamics. 
Table~\ref{tab:Chordal_Params} details the settings for the Chordal distance metric under both synchronous and asynchronous settings, while Table~\ref{tab:Geodesic_Params} lists the parameters for the Geodesic distance metric. Finally, Table~\ref{tab:Delay_Params} presents the parameter adjustments required for extreme communication delays.

\begin{table}[H]  
    \centering
    \renewcommand{\arraystretch}{1.3} 
    \setlength{\tabcolsep}{0.8em}     
    \caption{Hyperparameter settings for benchmark datasets using the Chordal distance metric. The table distinguishes between synchronous and asynchronous communication protocols.}
    \begin{tabular}{|c||c|c|c||c|c|c|}
        \hline
        \multirow{2}{*}{\textbf{Dataset}} & \multicolumn{3}{c||}{\textbf{Synchronous}} & \multicolumn{3}{c|}{\textbf{Asynchronous}} \\
        \cline{2-7}
        & $d$ & $m$ & $\Delta t$ & $d$ & $m$ & $\Delta t$ \\
        \hline
        {\sf Sm. Grid} & 1.5 & 1& 0.65 &4 &0.7 &0.1 \\
        \hline
        {\sf Sphere} &2  &0.7 &1 &4 &0.7 &0.1 \\
        \hline

        {\sf Torus} &2 &0.7 &1 & 4& 0.7& 0.1\\
        \hline
        {\sf Grid} & 2&0.7 & 1& 4& 0.7& 0.1\\
        \hline
        {\sf Cubicle} & 2& 0.7 & 1 & 4& 0.7& 0.1\\
        \hline
        {\sf Rim} & 2& 0.8& 1& 4& 0.7&0.1 \\
        \hline
        {\sf Garage} &2 & 0.8 & 1 & 3& 0.7&0.1 \\
        \hline
    \end{tabular}
    \label{tab:Chordal_Params}
\end{table}

\begin{table}[H] 
    \centering
    \renewcommand{\arraystretch}{1.3}
    \setlength{\tabcolsep}{1.5em}
    \caption{Hyperparameter settings for benchmark datasets using the Geodesic distance metric.}
    \begin{tabular}{|c||c|c|c|}
        \hline
        \textbf{Dataset} & $d$ & $m$ & $\Delta t$ \\
        \hline
        {\sf Sm. Grid} &4 &0.7 &0.7 \\
        \hline
        {\sf Sphere} &2 &0.7 &0.7 \\
        \hline
        {\sf Torus} &2 &0.7 &0.7 \\
        \hline
        {\sf Grid} &2 &0.7 &1 \\
        \hline
        {\sf Cubicle} & 3& 0.7& 0.7\\
        \hline
        {\sf Rim} & 6&1 &0.65 \\
        \hline
        {\sf Garage} & 2&0.7 &0.7 \\
        \hline
    \end{tabular}
    \label{tab:Geodesic_Params}
\end{table}

\begin{table}[H]
    \centering
    \renewcommand{\arraystretch}{1.3}
    \setlength{\tabcolsep}{1.5em} 
    \caption{Hyperparameter adjustments for the asynchronous delay robustness analysis. As the communication delay (Delay step) increases, the time step $\Delta t$ is adjusted to maintain system stability.}
    \begin{tabular}{|c||c|c|c|}
        \hline
        \textbf{Delay Step} & $d$ & $m$ & $\Delta t$ \\
        \hline
        3 & 5&0.45 &0.12 \\
        \hline
        7 & 5&0.45 & 0.075\\
        \hline
        10 &5 &0.45 &0.05 \\
        \hline
    \end{tabular}
    \label{tab:Delay_Params}
\end{table}

\makeatletter
\global\c@NAT@ctr=0
\makeatother

{\small
\bibliographystyleapp{unsrtnat} 
\bibliographyapp{appendix_references} }

%% file: main.bbl
\begin{thebibliography}{3}
\providecommand{\natexlab}[1]{#1}
\providecommand{\url}[1]{\texttt{#1}}
\expandafter\ifx\csname urlstyle\endcsname\relax
  \providecommand{\doi}[1]{doi: #1}\else
  \providecommand{\doi}{doi: \begingroup \urlstyle{rm}\Url}\fi

\bibitem[Bloch et~al.(1996)Bloch, Krishnaprasad, Marsden, and Ratiu]{bloch1996euler_app}
Anthony Bloch, PS~Krishnaprasad, Jerrold~E Marsden, and Tudor~S Ratiu.
\newblock The euler-poincar{\'e} equations and double bracket dissipation.
\newblock \emph{Communications in mathematical physics}, 175\penalty0 (1):\penalty0 1--42, 1996.

\bibitem[Lee et~al.(2007)Lee, Leok, and McClamroch]{lee2007lie}
Taeyoung Lee, Melvin Leok, and N~Harris McClamroch.
\newblock Lie group variational integrators for the full body problem.
\newblock \emph{Computer Methods in Applied Mechanics and Engineering}, 196\penalty0 (29-30):\penalty0 2907--2924, 2007.

\bibitem[Cendra et~al.(1999)Cendra, Holm, Marsden, and Ratiu]{cendra1999lagrangian}
Hern{\'a}n Cendra, Darryl~D Holm, Jerrold~E Marsden, and Tudor~S Ratiu.
\newblock Lagrangian reduction, the euler--poincaré equations, and semidirect products.
\newblock \emph{arXiv preprint chao-dyn/9906004}, 1999.

\end{thebibliography}


\begin{thebibliography}{42}
\providecommand{\natexlab}[1]{#1}
\providecommand{\url}[1]{\texttt{#1}}
\expandafter\ifx\csname urlstyle\endcsname\relax
  \providecommand{\doi}[1]{doi: #1}\else
  \providecommand{\doi}{doi: \begingroup \urlstyle{rm}\Url}\fi

\bibitem[Ebadi et~al.(2024)Ebadi, Bernreiter, Biggie, Catt, Chang, Chatterjee, Denniston, Deschênes, Harlow, Khattak, Nogueira, Palieri, Petráček, Petrlík, Reinke, Krátký, Zhao, Agha-mohammadi, Alexis, Heckman, Khosoussi, Kottege, Morrell, Hutter, Pauling, Pomerleau, Saska, Scherer, Siegwart, Williams, and Carlone]{ebadi_present_2024}
Kamak Ebadi, Lukas Bernreiter, Harel Biggie, Gavin Catt, Yun Chang, Arghya Chatterjee, Christopher~E. Denniston, Simon-Pierre Deschênes, Kyle Harlow, Shehryar Khattak, Lucas Nogueira, Matteo Palieri, Pavel Petráček, Matěj Petrlík, Andrzej Reinke, Vít Krátký, Shibo Zhao, Ali-akbar Agha-mohammadi, Kostas Alexis, Christoffer Heckman, Kasra Khosoussi, Navinda Kottege, Benjamin Morrell, Marco Hutter, Fred Pauling, François Pomerleau, Martin Saska, Sebastian Scherer, Roland Siegwart, Jason~L. Williams, and Luca Carlone.
\newblock Present and {Future} of {SLAM} in {Extreme} {Environments}: {The} {DARPA} {SubT} {Challenge}.
\newblock \emph{IEEE Trans. on Robotics}, 40:\penalty0 936--959, 2024.
\newblock ISSN 1941-0468.
\newblock \doi{10.1109/TRO.2023.3323938}.

\bibitem[Tian et~al.(2022)Tian, Chang, Arias, Nieto-Granda, How, and Carlone]{tian2022kimera}
Yulun Tian, Yun Chang, Fernando~Herrera Arias, Carlos Nieto-Granda, Jonathan~P How, and Luca Carlone.
\newblock Kimera-multi: Robust, distributed, dense metric-semantic slam for multi-robot systems.
\newblock \emph{IEEE Transactions on Robotics}, 38\penalty0 (4), 2022.

\bibitem[Liu et~al.(2024)Liu, Lei, Prabhu, Tao, Spasojevic, Chaudhari, Atanasov, and Kumar]{liu_slideslam_2024}
Xu~Liu, Jiuzhou Lei, Ankit Prabhu, Yuezhan Tao, Igor Spasojevic, Pratik Chaudhari, Nikolay Atanasov, and Vijay Kumar.
\newblock {SlideSLAM}: {Sparse}, {Lightweight}, {Decentralized} {Metric}-{Semantic} {SLAM} for {Multi}-{Robot} {Navigation}, July 2024.
\newblock arXiv:2406.17249.

\bibitem[Lajoie and Beltrame(2024)]{lajoie_swarm-slam_2024}
Pierre-Yves Lajoie and Giovanni Beltrame.
\newblock Swarm-{SLAM} : {Sparse} {Decentralized} {Collaborative} {Simultaneous} {Localization} and {Mapping} {Framework} for {Multi}-{Robot} {Systems}.
\newblock \emph{IEEE Robotics and Automation Letters}, 9(1):\penalty0 475--482, January 2024.
\newblock ISSN 2377-3766, 2377-3774.
\newblock \doi{10.1109/LRA.2023.3333742}.

\bibitem[Schmuck et~al.(2021)Schmuck, Ziegler, Karrer, Perraudin, and Chli]{schmuck_covins_2021}
Patrik Schmuck, Thomas Ziegler, Marco Karrer, Jonathan Perraudin, and Margarita Chli.
\newblock {COVINS}: {Visual}-{Inertial} {SLAM} for {Centralized} {Collaboration}.
\newblock In \emph{2021 {IEEE} {Int.} {Symposium} on {Mixed} and {Augmented} {Reality} {Adjunct}}, pages 171--176, October 2021.
\newblock \doi{10.1109/ISMAR-Adjunct54149.2021.00043}.

\bibitem[Tian et~al.(2021)Tian, Khosoussi, Rosen, and How]{tian_distributed_2021}
Yulun Tian, Kasra Khosoussi, David~M. Rosen, and Jonathan~P. How.
\newblock Distributed {Certifiably} {Correct} {Pose}-{Graph} {Optimization}.
\newblock \emph{IEEE Trans. on Robotics}, 37(6):\penalty0 2137--2156, December 2021.
\newblock ISSN 1941-0468.
\newblock \doi{10.1109/TRO.2021.3072346}.

\bibitem[McGann et~al.(2024)McGann, Lassak, and Kaess]{mcgann_asynchronous_2024}
Daniel McGann, Kyle Lassak, and Michael Kaess.
\newblock Asynchronous {Distributed} {Smoothing} and {Mapping} via {On}-{Manifold} {Consensus} {ADMM}.
\newblock In \emph{IEEE Int. Conf. on Robotics and Automation}, pages 4577--4583, 2024.

\bibitem[Fan and Murphey(2024)]{fan_majorization_2024}
Taosha Fan and Todd~D. Murphey.
\newblock Majorization {Minimization} {Methods} for {Distributed} {Pose} {Graph} {Optimization}.
\newblock \emph{IEEE Trans.\ on Robotics}, 40:\penalty0 22--42, 2024.
\newblock ISSN 1941-0468.
\newblock \doi{10.1109/TRO.2023.3324818}.

\bibitem[Murai et~al.(2024)Murai, Ortiz, Saeedi, Kelly, and Davison]{murai_robot_2024}
Riku Murai, Joseph Ortiz, Sajad Saeedi, Paul H.~J. Kelly, and Andrew~J. Davison.
\newblock A {Robot} {Web} for {Distributed} {Many}-{Device} {Localization}.
\newblock \emph{IEEE Trans. on Robotics}, 40:\penalty0 121--138, 2024.
\newblock ISSN 1941-0468.
\newblock \doi{10.1109/TRO.2023.3324127}.

\bibitem[Tian et~al.(2020)Tian, Koppel, Bedi, and How]{tian_asynchronous_2020}
Yulun Tian, Alec Koppel, Amrit~Singh Bedi, and Jonathan~P. How.
\newblock Asynchronous and {Parallel} {Distributed} {Pose} {Graph} {Optimization}.
\newblock \emph{IEEE Robotics and Automation Letters}, 5(4):\penalty0 5819--5826, October 2020.
\newblock ISSN 2377-3766.
\newblock \doi{10.1109/LRA.2020.3010216}.

\bibitem[Bloch et~al.(1996)Bloch, Krishnaprasad, Marsden, and Ratiu]{bloch1996euler}
Anthony Bloch, PS~Krishnaprasad, Jerrold~E Marsden, and Tudor~S Ratiu.
\newblock The euler-poincar{\'e} equations and double bracket dissipation.
\newblock \emph{Communications in mathematical physics}, 175\penalty0 (1):\penalty0 1--42, 1996.

\bibitem[Duong et~al.(2024)Duong, Altawaitan, Stanley, and Atanasov]{duong2024port}
Thai Duong, Abdullah Altawaitan, Jason Stanley, and Nikolay Atanasov.
\newblock Port-hamiltonian neural ode networks on lie groups for robot dynamics learning and control.
\newblock \emph{IEEE Transactions on Robotics}, 2024.

\bibitem[Teng et~al.(2025)Teng, Jasour, Vasudevan, and Ghaffari]{teng2025convex}
Sangli Teng, Ashkan Jasour, Ram Vasudevan, and Maani Ghaffari.
\newblock Convex geometric motion planning of multi-body systems on lie groups via variational integrators and sparse moment relaxation.
\newblock \emph{The International Journal of Robotics Research}, 44\penalty0 (10-11):\penalty0 1784--1813, 2025.

\bibitem[Cunningham et~al.(2010)Cunningham, Paluri, and Dellaert]{cunningham_ddf-sam_2010}
Alexander Cunningham, Manohar Paluri, and Frank Dellaert.
\newblock {DDF}-{SAM}: {Fully} distributed {SLAM} using {Constrained} {Factor} {Graphs}.
\newblock In \emph{2010 {IEEE}/{RSJ} {Int.} {Conf.} on {Intelligent} {Robots} and {Systems}}, pages 3025--3030, 2010.
\newblock \doi{10.1109/IROS.2010.5652875}.

\bibitem[Cunningham et~al.(2013)Cunningham, Indelman, and Dellaert]{cunningham_ddf-sam_2013}
Alexander Cunningham, Vadim Indelman, and Frank Dellaert.
\newblock {DDF}-{SAM} 2.0: {Consistent} distributed smoothing and mapping.
\newblock In \emph{2013 {IEEE} {Int.} {Conf.} on {Robotics} and {Automation}}, pages 5220--5227, 2013.
\newblock \doi{10.1109/ICRA.2013.6631323}.

\bibitem[Choudhary et~al.(2017)Choudhary, Carlone, Nieto, Rogers, Christensen, and Dellaert]{choudhary_distributed_2017}
Siddharth Choudhary, Luca Carlone, Carlos Nieto, John Rogers, Henrik~I Christensen, and Frank Dellaert.
\newblock Distributed mapping with privacy and communication constraints: {Lightweight} algorithms and object-based models.
\newblock \emph{The Int. Journal of Robotics Research}, 36(12):\penalty0 1286--1311, October 2017.
\newblock ISSN 0278-3649, 1741-3176.
\newblock \doi{10.1177/0278364917732640}.

\bibitem[Rosen et~al.(2019)Rosen, Carlone, Bandeira, and Leonard]{rosen_se_sync_2019}
David~M Rosen, Luca Carlone, Afonso~S Bandeira, and John~J Leonard.
\newblock {SE-Sync}: A certifiably correct algorithm for synchronization over the special {Euclidean} group.
\newblock \emph{The Int. Journal of Robotics Research}, 38(2-3):\penalty0 95--125, 2019.
\newblock \doi{10.1177/0278364918784361}.

\bibitem[Briales and Gonzalez-Jimenez(2017)]{briales2017cartan}
Jesus Briales and Javier Gonzalez-Jimenez.
\newblock Cartan-sync: Fast and global se (d)-synchronization.
\newblock \emph{IEEE Robotics and Automation Letters}, 2\penalty0 (4):\penalty0 2127--2134, 2017.

\bibitem[Li et~al.(2024)Li, Guo, Yi, and Hong]{li_distributed_2024}
Cunhao Li, Guanghui Guo, Peng Yi, and Yiguang Hong.
\newblock Distributed {Pose}-{Graph} {Optimization} {With} {Multi}-{Level} {Partitioning} for {Multi}-{Robot} {SLAM}.
\newblock \emph{IEEE Robotics and Automation Letters}, 9(6):\penalty0 4926--4933, June 2024.
\newblock ISSN 2377-3766.
\newblock \doi{10.1109/LRA.2024.3382531}.

\bibitem[McGann and Kaess(2024)]{mcgann_imesa_2024}
D.~McGann and M.~Kaess.
\newblock {iMESA}: Incremental distributed optimization for collaborative simultaneous localization and mapping.
\newblock In \emph{Proc. Robotics: Science and Systems (RSS)}, Delft, {NL}, 2024.

\bibitem[B{\"a}nninger et~al.(2023)B{\"a}nninger, Alzugaray, Karrer, and Chli]{banninger2023cross}
Philipp B{\"a}nninger, Ignacio Alzugaray, Marco Karrer, and Margarita Chli.
\newblock Cross-agent relocalization for decentralized collaborative slam.
\newblock In \emph{2023 IEEE International Conference on Robotics and Automation (ICRA)}, pages 5551--5557. IEEE, 2023.

\bibitem[Liu and Chli(2026)]{liu2026cbs}
Xiangyu Liu and Margarita Chli.
\newblock Distributed pose graph optimization via contractive belief sharing.
\newblock \emph{IEEE Robotics and Automation Letters}, 2026.

\bibitem[Su et~al.(2016)Su, Boyd, and Candes]{su2016differential}
Weijie Su, Stephen Boyd, and Emmanuel~J Candes.
\newblock A differential equation for modeling nesterov's accelerated gradient method: Theory and insights.
\newblock \emph{Journal of Machine Learning Research}, 2016.

\bibitem[Nesterov(1983)]{nesterov1983method}
Yurii Nesterov.
\newblock A method for solving the convex programming problem with convergence rate o (1/k2).
\newblock In \emph{Soviet Mathematics Doklady}, 1983.

\bibitem[Wibisono et~al.(2016)Wibisono, Wilson, and Jordan]{wibisono2016variational}
Andre Wibisono, Ashia~C Wilson, and Michael~I Jordan.
\newblock A variational perspective on accelerated methods in optimization.
\newblock \emph{proceedings of the National Academy of Sciences}, 113\penalty0 (47), 2016.

\bibitem[Wilson et~al.(2021)Wilson, Recht, and Jordan]{wilson2021lyapunov}
Ashia~C Wilson, Ben Recht, and Michael~I Jordan.
\newblock A lyapunov analysis of accelerated methods in optimization.
\newblock \emph{Journal of Machine Learning Research}, pages 1--34, 2021.

\bibitem[Betancourt et~al.(2018)Betancourt, Jordan, and Wilson]{betancourt2018symplectic}
Michael Betancourt, Michael~I Jordan, and Ashia~C Wilson.
\newblock On symplectic optimization.
\newblock \emph{arXiv preprint arXiv:1802.03653}, 2018.

\bibitem[Duruisseaux and Leok(2023)]{duruisseaux2023practical}
Valentin Duruisseaux and Melvin Leok.
\newblock Practical perspectives on symplectic accelerated optimization.
\newblock \emph{Optimization Methods and Software}, 38\penalty0 (6):\penalty0 1230--1268, 2023.

\bibitem[Hairer et~al.(2006)Hairer, Lubich, and Wanner]{hairer2006structure}
Ernst Hairer, Christian Lubich, and Gerhard Wanner.
\newblock Structure-preserving algorithms for ordinary differential equations.
\newblock \emph{Geometric numerical integration}, 31, 2006.

\bibitem[Hauswirth et~al.(2016)Hauswirth, Bolognani, Hug, and D{\"o}rfler]{hauswirth2016projected}
Adrian Hauswirth, Saverio Bolognani, Gabriela Hug, and Florian D{\"o}rfler.
\newblock Projected gradient descent on riemannian manifolds with applications to online power system optimization.
\newblock In \emph{2016 54th Annual Allerton Conference on Communication, Control, and Computing (Allerton)}. IEEE, 2016.

\bibitem[Duruisseaux and Leok(2022)]{duruisseaux2022accelerated}
Valentin Duruisseaux and Melvin Leok.
\newblock Accelerated optimization on riemannian manifolds via discrete constrained variational integrators.
\newblock \emph{Journal of Nonlinear Science}, 32\penalty0 (4):\penalty0 42, 2022.

\bibitem[Golyanik et~al.(2016)Golyanik, Ali, and Stricker]{golyanik2016gravitational}
Vladislav Golyanik, Sk~Aziz Ali, and Didier Stricker.
\newblock Gravitational approach for point set registration.
\newblock In \emph{Proceedings of the IEEE conference on computer vision and pattern recognition}, 2016.

\bibitem[Jauer et~al.(2018)Jauer, Kuhlemann, Bruder, Schweikard, and Ernst]{jauer2018efficient}
Philipp Jauer, Ivo Kuhlemann, Ralf Bruder, Achim Schweikard, and Floris Ernst.
\newblock Efficient registration of high-resolution feature enhanced point clouds.
\newblock \emph{IEEE transactions on pattern analysis and machine intelligence}, 41\penalty0 (5):\penalty0 1102--1115, 2018.

\bibitem[Golyanik et~al.(2019)Golyanik, Theobalt, and Stricker]{golyanik2019accelerated}
Vladislav Golyanik, Christian Theobalt, and Didier Stricker.
\newblock Accelerated gravitational point set alignment with altered physical laws.
\newblock In \emph{Proceedings of the IEEE/CVF International Conference on Computer Vision}, pages 2080--2089, 2019.

\bibitem[Ali et~al.(2018)Ali, Golyanik, and Stricker]{ali2018nrga}
Sk~Aziz Ali, Vladislav Golyanik, and Didier Stricker.
\newblock Nrga: Gravitational approach for non-rigid point set registration.
\newblock In \emph{2018 International Conference on 3D Vision (3DV)}, pages 756--765. IEEE, 2018.

\bibitem[Zhao et~al.(2022)Zhao, Ma, Jia, Yan, and Huang]{zhao2022graphreg}
Mingyang Zhao, Lei Ma, Xiaohong Jia, Dong-Ming Yan, and Tiejun Huang.
\newblock Graphreg: Dynamical point cloud registration with geometry-aware graph signal processing.
\newblock \emph{IEEE Transactions on Image Processing}, 31:\penalty0 7449--7464, 2022.

\bibitem[Yang et~al.(2021)Yang, Doran, and Slotine]{yang2021dynamical}
Heng Yang, Chris Doran, and Jean-Jacques Slotine.
\newblock Dynamical pose estimation.
\newblock In \emph{Proceedings of the IEEE/CVF International Conference on Computer Vision}, pages 5926--5935, 2021.

\bibitem[Dellaert and {GTSAM Contributors}(2022)]{gtsam}
Frank Dellaert and {GTSAM Contributors}.
\newblock borglab/gtsam, May 2022.
\newblock URL \url{https://github.com/borglab/gtsam}.

\bibitem[Absil et~al.(2008)Absil, Mahony, and Sepulchre]{absil2008optimization}
P-A Absil, Robert Mahony, and Rodolphe Sepulchre.
\newblock \emph{Optimization algorithms on matrix manifolds}.
\newblock Princeton University Press, 2008.

\bibitem[Boumal et~al.(2019)Boumal, Absil, and Cartis]{boumal2019global}
Nicolas Boumal, Pierre-Antoine Absil, and Coralia Cartis.
\newblock Global rates of convergence for nonconvex optimization on manifolds.
\newblock \emph{IMA Journal of Numerical Analysis}, 39\penalty0 (1):\penalty0 1--33, 2019.

\bibitem[Lian et~al.(2018)Lian, Zhang, Zhang, and Liu]{lian2018asynchronous}
Xiangru Lian, Wei Zhang, Ce~Zhang, and Ji~Liu.
\newblock Asynchronous decentralized parallel stochastic gradient descent.
\newblock In \emph{International conference on machine learning}, pages 3043--3052. PMLR, 2018.

\bibitem[Dolan and Mor{\'e}(2002)]{dolan2002benchmarking}
Elizabeth~D Dolan and Jorge~J Mor{\'e}.
\newblock Benchmarking optimization software with performance profiles.
\newblock \emph{Mathematical programming}, 91\penalty0 (2):\penalty0 201--213, 2002.

\end{thebibliography}
